\documentclass[10pt]{article}
\usepackage[preprint]{tmlr}

\usepackage{amsmath,amssymb,amsthm,mathtools}
\usepackage{bm}
\usepackage{booktabs}
\usepackage{array}
\usepackage{enumitem}
\usepackage{microtype}
\usepackage{url}
\usepackage{hyperref}
\usepackage[nameinlink,noabbrev]{cleveref}

\theoremstyle{plain}
\newtheorem{theorem}{Theorem}[section]
\newtheorem{proposition}[theorem]{Proposition}
\newtheorem{lemma}[theorem]{Lemma}
\newtheorem{corollary}[theorem]{Corollary}
\theoremstyle{definition}
\newtheorem{definition}[theorem]{Definition}

\newtheorem{example}[theorem]{Example}
\theoremstyle{remark}
\newtheorem{remark}[theorem]{Remark}

\newcommand{\R}{\mathbb R}
\newcommand{\E}{\mathbb E}
\newcommand{\SPD}{\mathbb S_{++}}
\newcommand{\PSD}{\mathbb S_{+}}
\newcommand{\Sym}{\mathbb S}
\newcommand{\AI}{\mathrm{AI}}
\newcommand{\Gr}{\mathrm{Gr}}
\newcommand{\St}{\mathrm{St}}
\newcommand{\tr}{\operatorname{tr}}
\newcommand{\diag}{\operatorname{diag}}
\newcommand{\blockdiag}{\operatorname{blockdiag}}
\newcommand{\rank}{\operatorname{rank}}
\newcommand{\range}{\operatorname{range}}
\newcommand{\ri}{\operatorname{ri}}
\newcommand{\cone}{\operatorname{cone}}
\newcommand{\Ann}{\operatorname{Ann}}
\newcommand{\prox}{\operatorname{prox}}
\newcommand{\argminop}{\operatorname*{arg\,min}}
\newcommand{\opnorm}[1]{\left\lVert #1\right\rVert_{\operatorname{op}}}
\newcommand{\fro}[1]{\left\lVert #1\right\rVert_F}
\newcommand{\norm}[1]{\left\lVert #1\right\rVert}

\newcommand{\dd}{\mathrm d}
\newcommand{\dai}{d_{\AI}}
\newcommand{\dgr}{d_{\Gr}}
\newcommand{\Cpl}{\mathfrak C}
\newcommand{\Xiop}{\Xi}

\newcommand{\Rhat}{\widehat R}
\newcommand{\Phat}{\widehat P}
\newcommand{\Sigmahat}{\widehat\Sigma}

\newcommand{\Fcal}{\mathcal F}
\newcommand{\Kcal}{\mathcal K}

\crefname{assumption}{Assumption}{Assumptions}
\crefname{definition}{Definition}{Definitions}
\crefname{example}{Example}{Examples}
\crefname{theorem}{Theorem}{Theorems}
\crefname{proposition}{Proposition}{Propositions}
\crefname{lemma}{Lemma}{Lemmas}
\crefname{corollary}{Corollary}{Corollaries}

\title{Information-Induced Training Geometry:\\
Exact Reduction, Canonical Completion, and Structured Expressivity}

\author{\name Zavier Li \email zavierli888@gmail.com\\
      \addr Xidian University\\
      \addr Xi'an, China}

\def\month{07}
\def\year{2026}
\def\openreview{\url{}}

\begin{document}
\maketitle

\begin{abstract}
Training data constrains optimizer geometry through the covectors visible to
a declared information channel.  We study how such partial information
determines a full positive cometric relative to a reference and which degrees
of freedom remain unidentified.  Our central result resolves the
full-column-rank SPD channel
\(\pi_A(P)=A^\top P A\) under affine-invariant Riemannian geometry.  The map
is a split-Hadamard metric submetry and admits an explicit unique completion
\(\Psi_{P_0,A}\), which is the AIRM-nearest full geometry realizing a visible
target and yields exact full-to-visible variational reduction.  When the
channel moves, the completions form a gauge-invariant rank stratification of
the SPD cone.  Its closed-form pullback pair metric separates visible-metric
motion from subspace rotation through the mismatch weight
\(R+R^{-1}-2I\), yields an explicit positive-semidefinite multi-direction Gram
matrix, and exposes the precise singularity of reference-valued modes.  The
mechanism is explained by a metric theorem equating ball
submetry, attained fiber distance, and lossless reduction of every monotone
radial visible decision problem; a smooth split-Hadamard theorem supplies
coherent information sheets, proximal commutation, and solution-wise
gradient-flow lifting.  The SPD realization also gives closed-form
prior--data shrinkage.  Diagonal
and block optimizer families then reduce to relative-interior conic image
tests with valid facial certificates, while deterministic and finite-sample
bounds quantify recovery of the visible geometry and its subspace.  Together
these results characterize exact reduction, reference-dependent completion,
and structured expressivity for the stated finite-dimensional AIRM model.
\end{abstract}

\section{Introduction}
\label{sec:introduction}

Adaptive optimization is often described as a choice of geometry.  Natural
gradient uses the Fisher metric \citep{amari1998natural}; K-FAC and related
methods use structured Fisher or Gauss--Newton approximations
\citep{martens2015kfac,george2018ekfac}; AdaGrad and Adam accumulate diagonal
statistics \citep{duchi2011adagrad,kingma2014adam}; and Shampoo-type methods
use matrix or tensor preconditioners
\citep{gupta2018shampoo,morwani2024shampoo}.  These methods differ in how they
construct and restrict a positive map from covectors to parameter directions.
Before comparing those constructions, one must decide what geometric
information the current data can actually see.

Fix a parameter point \(\theta\), let \(V=T_\theta\mathcal M\) and
\(E=V^*\), and write the per-example loss differentials as
\(a_i=\dd_\theta\ell_i\in E\).  Their span
\[
  W=\operatorname{span}\{a_1,\ldots,a_n\}\subset E
\]
is the first-order visible cotangent space.  Its annihilator
\(\Ann(W)\subset V\) is invisible to every example loss at that point.  A
declared positive visible target \(R\in\SPD(W)\) constrains the restriction of
a full cometric \(P\in\SPD(E)\), but it leaves an entire fiber
\[
  \{P\in\SPD(E):P|_{W\times W}=R\}
\]
undetermined.  The central question is geometric:

\begin{quote}
Which observation maps admit lossless visible reduction, when do they possess
a canonical shortest completion, and what does that completion look like for
training cometrics?
\end{quote}

This question has three levels.  At the first, no linearity or smoothness is
available.  We show that exact reduction of every reference-radial visible
decision problem characterizes metric submetries.  This maximal statement
isolates the precise property required for lossless completion.

At the second level, the spaces are Hadamard manifolds and the information map
is a Riemannian submersion.  If its horizontal distribution integrates to
complete totally geodesic leaves, each leaf is a global isometric information
sheet.  The resulting section is coherent across updates, is the unique
shortest lift, and commutes exactly with proximal maps and gradient flows.

At the third level, the ambient geometry is the SPD cone with the
affine-invariant Riemannian metric (AIRM), and visible information is a
full-column-rank linear channel
\[
  \pi_A(P)=A^\top P A.
\]
This realization is fully explicit.  For reference \(P_0\), set
\(R_0=A^\top P_0A\) and
\(U_0=P_0^{1/2}AR_0^{-1/2}\).  The canonical completion of
\(R\in\SPD^m\) is
\begin{equation}
  \Psi_{P_0,A}(R)
  =P_0^{1/2}
  \left[I+U_0(R_0^{-1/2}RR_0^{-1/2}-I)U_0^\top\right]
  P_0^{1/2}.
  \label{eq:intro-completion}
\end{equation}
It is the unique AIRM-nearest point in the visible fiber, and the full
deformation cost equals the visible cost exactly.

The explicit realization reveals additional structure.  Modulo visible
coordinate changes, active information descriptors \((A,R)\) stratify the
entire SPD cone by \(\rank(P-P_0)\).  In normalized horizontal coordinates,
the exact infinitesimal metric is
\begin{equation}
  \norm{\dot P}_{P,\AI}^2
  =
  \fro{R^{-1/2}\dot R R^{-1/2}}^2
  +2\tr\!\left(K[R+R^{-1}-2I]K^\top\right).
  \label{eq:intro-line-element}
\end{equation}
The first term measures visible geometry change.  The second measures rotation
of the visible subspace.  Reference-valued modes have zero rotation cost and
are precisely the singular directions of the descriptor.  Polarizing
\eqref{eq:intro-line-element} gives a closed-form pair metric for arbitrary
coupled perturbation directions.  Its multi-direction matrix is a positive
semidefinite Gram matrix.  This metric cross-term measures the AIRM geometry
of completed motion; it is distinct from the Hessian of the reduced
minimum-distance value.

\paragraph{Contributions.}
The fixed-channel SPD theorem and the moving-channel pair metric are the two
central results.  The abstract, structured, and statistical theorems identify
their mechanism and consequences.
\begin{enumerate}[leftmargin=*,itemsep=2pt]
  \item \emph{Exact SPD reduction and completion.}  We prove that every
  full-column-rank SPD compression is a
  split-Hadamard submetry and derive the closed-form completion
  \eqref{eq:intro-completion}, its horizontal geometry, exact full-to-visible
  decision reduction, and closed-form prior--data shrinkage.  A
  radius-conditioned minimax identity quantifies the remaining invisible
  ambiguity.
  \item \emph{Moving-channel geometry.}  We identify the gauge quotient and
  all rank strata of completed SPD geometry, and derive the closed-form
  mismatch-weighted pair metric whose diagonal is
  \eqref{eq:intro-line-element}.  Its Gram factorization gives exact rank,
  metric-orthogonality and differential-kernel criteria, together with
  moving-reference covariance and the rank-loss boundary.
  \item \emph{General mechanism.}  We characterize metric submetries by
  attained fiber distance and universal radial-visible reduction.  Complete
  horizontal leaves on Hadamard manifolds yield coherent isometric sections,
  exact proximal commutation, and solution-wise gradient-flow lifting.
  \item \emph{Implementability and recovery.}  We characterize diagonal and
  block structured expressivity by
  relative interiors of conic images, with strict separation and facial
  certificates for distinct failure modes.  Deterministic and finite-sample
  bounds then separate visible AIRM error from Grassmann subspace error.
\end{enumerate}

\paragraph{Scope.}
The theory starts after a positive visible target has been declared.  Fisher,
Gauss--Newton, gradient covariance, damping, and user-specified audit targets
provide different constructions of such a target; the paper does not identify
them with one another.  The moving-reference formula is a kinematic identity,
not an optimizer evolution law.  Likewise, exact reduction determines a
canonical full representative of a visible decision but does not prescribe
which visible decision an optimizer should make.

\paragraph{Organization.}
\Cref{sec:related} positions the work.  \Cref{sec:metric-foundations}
establishes the maximal metric theorem, and \cref{sec:split-hadamard} gives its
smooth global form.  \Cref{sec:spd-realization,sec:stratified-geometry}
develop the explicit SPD realization and its rank-stratified pair metric.
\Cref{sec:structured-expressivity} treats structured families, and
\cref{sec:statistical-recovery} develops perturbation and sampling results.

\section{Related Work}
\label{sec:related}

\paragraph{Riemannian submersions and submetries.}
The differential geometry of Riemannian submersions starts from the
horizontal--vertical equations of \citet{oneill1966submersion}.  Metric
submetries provide the corresponding ball-surjectivity language;
\citet{berestovskii2000metric} characterize Riemannian submersions through
this metric property, and \citet{guijarro2011submetries} delineate when
submetries recover smooth submersions.  Our abstract theorem uses this
classical metric structure and adds a variational characterization: exact
fiber distance is equivalent to lossless reduction for every monotone radial
visible objective.  The split-Hadamard result imposes global integrable
horizontal geometry to obtain coherent shortest sections, while the SPD
compression theorem verifies this structure and resolves the section in
closed form.

\paragraph{Information geometry and natural gradient.}
Natural gradient equips statistical model space with Fisher geometry
\citep{amari1998natural}.  Structured Fisher and Gauss--Newton approximations
motivate K-FAC, EK-FAC, and related methods
\citep{martens2015kfac,george2018ekfac,martens2020natural}; limitations of
identifying Fisher and Hessian geometry are discussed by
\citet{kunstner2019limitations}.  Mirror descent also admits an
information-geometric interpretation \citep{raskutti2015information}.  Our
focus is the observation layer beneath a chosen training geometry: which
cotangent information is visible, which ambient completions are compatible,
and which completion is selected by a reference.

Sufficient statistics contract information-geometric tensors and preserve
them under stronger sufficiency conditions \citep{ay2015sufficient}.  That
work concerns statistical models and Fisher-type tensors.  Here the declared
object is a finite-dimensional positive cometric, the observation is the
linear compression \(A^\top PA\), and the main question is its
reference-dependent AIRM completion.  Statistical sufficiency motivates the
visibility language, while the completion and moving-channel line element
address a different geometric problem.
The AIRM covariance factor is also the Fisher geometry of the multivariate
normal family up to convention-dependent scaling
\citep{lovric2000multivariate}; this explains the Gaussian interpretation but
does not determine the reference-relative completion of an arbitrary declared
visible cometric.

\paragraph{SPD geometry and matrix completion.}
The affine-invariant geometry of positive matrices is classical
\citep{bhatia2007positive,pennec2006riemannian}.  Positive definite and
covariance completion usually infer missing matrix entries from sparsity or
graph constraints \citep{dempster1972covariance,grone1984positive}.  Here the
known object is the pullback \(A^\top P A\) on an information channel, and the
missing object is the entire invisible fiber.  The selected completion is
defined by AIRM distance to a reference and is characterized globally as a
horizontal section of a metric submetry.

Totally geodesic AIRM submanifolds and their nearest-point projections have
been characterized directly inside the SPD manifold
\citep{tumpach2024totally}.  Our information sheet is a particular
congruence-transformed block submanifold of that kind.  The additional object
here is the transverse affine-compression fiber: we identify its unique
nearest point to every reference, prove that the compression is a metric
submetry, and derive universal radial-visible reduction.  Thus total geodesy
supports the construction but does not by itself supply the fiber-completion
or variational statements.

Quotient geometries for fixed-rank positive-semidefinite matrices have also
been developed for matrix means and optimization
\citep{bonnabel2010fixedrank,vandereycken2013fixedrank,massart2020quotient}.
Those manifolds parameterize a rank-\(m\) PSD matrix.  Our stratum instead
fixes the rank of the displacement \(P-P_0\), which may be indefinite while
\(P\) remains positive definite, and its line element is the pullback of the
ambient AIRM at \(P_0+(P-P_0)\).  The resulting mismatch weight
\(R+R^{-1}-2I\) and its reference-valued kernel are specific to this
completion geometry.

\paragraph{Metric learning and relative-entropy projections.}
Information-theoretic metric learning uses LogDet or relative-entropy
projections under supervised constraints \citep{davis2007information}.  Our
Gaussian KL consequence is related, but the main object is the full fiber of
an information compression.  The metric-submetry theorem explains when every
radial visible decision, not only a particular divergence projection, reduces
without loss.  Classical information projections minimize divergence under
linear or convex constraints
\citep{csiszar1975idivergence,csiszar2003projections}.  Our Gaussian KL
corollary is a finite-dimensional instance with a closed-form fiber
projection; the AIRM shortest-completion and stratified line-element claims
are separate from the divergence-projection theory.

\paragraph{Adaptive and structured preconditioning.}
AdaGrad, Adam, Shampoo, and related methods impose diagonal, block, matrix, or
tensor structure on positive preconditioners
\citep{duchi2011adagrad,kingma2014adam,gupta2018shampoo,morwani2024shampoo}.
Optimal diagonal preconditioning and matrix-scaling viewpoints study related
coordinate restrictions \citep{qu2025optimaldiagonal}.  We formulate a
different certificate: whether a structured positive family can realize a
declared visible target.  The answer is a relative-interior condition in a
linear conic image, which distinguishes strict representability, boundary
failure, and closed-cone infeasibility.

\paragraph{Convex and semidefinite geometry.}
The structured certificates use finite-dimensional convex separation,
relative interiors, and semidefinite images
\citep{rockafellar1970convex,vandenberghe1996semidefinite,boyd2004convex}.
These tools enter only after the full-SPD information map has been resolved;
they quantify what is lost when the completion is restricted to an
implementable optimizer family.

\paragraph{Subspace perturbation and matrix concentration.}
When the visible subspace is estimated from a covariance or Gram matrix, its
recovery requires an eigengap and a principal-angle perturbation bound
\citep{yu2015davis}.  Covariance concentration then controls the ambient
operator error \citep{tropp2015matrix}.  We combine these ingredients with
the exact stratified line element to obtain a completed-geometry error bound
that separates visible SPD error from subspace rotation.

\section{Metric Foundations of Exact Information Reduction}
\label{sec:metric-foundations}

We first separate exact information reduction from smooth or matrix-specific
structure.  Let \((X,d_X)\) and \((Y,d_Y)\) be arbitrary metric spaces and let
\(q:X\to Y\) be surjective.

\begin{definition}[Metric submetry]
\label{def:metric-submetry}
The map \(q\) is a metric submetry if for every \(x\in X\) and \(r\ge0\),
\begin{equation}
  q(\overline B_X(x,r))=\overline B_Y(q(x),r).
  \label{eq:metric-submetry}
\end{equation}
\end{definition}

A submetry maps every closed ball onto the ball of the same radius.  It is
therefore stronger than a nonexpansive surjection: it guarantees that every
visible displacement has an ambient lift with no excess distance.

\begin{theorem}[Maximal exact information reduction]
\label{thm:metric-characterization}
The following statements are equivalent.
\begin{enumerate}[label=(\Alph*),leftmargin=*]
  \item \(q\) is a metric submetry.
  \item For every \(x\in X\) and \(y\in Y\), the fiber distance is attained
  and exact:
  \begin{equation}
    d_Y(q(x),y)=\min_{q(z)=y}d_X(x,z).
    \label{eq:exact-fiber-distance}
  \end{equation}
  \item For every nondecreasing
  \(\rho:[0,\infty)\to(-\infty,+\infty]\), every
  \(\phi:Y\to\R\cup\{+\infty\}\), and every \(x\in X\),
  \begin{equation}
    \boxed{
    \inf_{z\in X}\{\rho(d_X(x,z))+\phi(q(z))\}
    =
    \inf_{y\in Y}\{\rho(d_Y(q(x),y))+\phi(y)\}.}
    \label{eq:universal-radial-reduction}
  \end{equation}
\end{enumerate}
Thus metric submetries are exactly the observation maps for which every
reference-radial, visible-functional decision problem reduces without loss.
\end{theorem}

The proof, given in \cref{app:metric-proofs}, is elementary but sharp.  Hard
ball costs and singleton visible losses recover the full ball-surjectivity
condition, so the theorem cannot be enlarged while preserving universal exact
reduction.

The theorem determines optimal values but does not choose one lift when a
fiber contains several nearest points.  Canonical completion requires a
coherent family of sections.  Suppose \(q\) is \(1\)-Lipschitz and for every
\(x\in X\) a section \(s_x:Y\to X\) is given such that
\begin{equation}
  q\circ s_x=\operatorname{id}_Y,
  \qquad
  s_x(q(x))=x,
  \qquad
  s_{s_x(y)}=s_x.
  \label{eq:coherent-sections}
\end{equation}
The last identity says that moving inside one selected information sheet does
not change the sheet.

For a function \(F\) on a metric space, we use
\[
  e_\tau F(x)=\inf_z\left\{F(z)+\frac{d(x,z)^2}{2\tau}\right\}
\]
for its Moreau envelope and \(\prox_{\tau F}(x)\) for the corresponding
minimizer when it is unique.

\begin{theorem}[Coherent exact reduction]
\label{thm:coherent-reduction}
Under \eqref{eq:coherent-sections}, the following are equivalent:
\begin{enumerate}[label=(\Alph*),leftmargin=*]
  \item every \(s_x\) is an isometric embedding;
  \item for every \(x\in X\) and \(y\in Y\),
  \begin{equation}
    d_Y(q(x),y)
    =d_X(x,s_x(y))
    =\min_{q(z)=y}d_X(x,z);
    \label{eq:coherent-fiber-distance}
  \end{equation}
  \item the equality \eqref{eq:universal-radial-reduction} holds and
  \(s_x\) maps every reduced minimizer to a full minimizer.
\end{enumerate}
Whenever these conditions hold, \(q\) is a metric submetry.  If the nearest
point in every fiber is unique, the splitting is rigid.

If \(Y\) is complete CAT(0), \(\phi\) is proper, lower semicontinuous,
geodesically convex, and bounded below, and \(\tau>0\), then
\begin{equation}
  \boxed{
  \prox^{X,\mathrm{can}}_{\tau(\phi\circ q)}(x)
  =
  s_x\!\left(\prox^Y_{\tau\phi}(q(x))\right).}
  \label{eq:coherent-prox}
\end{equation}
Coherent proximal iterations are therefore exact lifts of the visible
iterations; under rigidity, the full proximal point is unique.
\end{theorem}

\begin{remark}[Three distinct requirements]
Surjectivity provides visible representatives.  Submetry provides exact
distance.  Coherence provides a reusable canonical sheet.  None of these
properties follows from the others without the stated assumptions.
\end{remark}

\section{Split-Hadamard Information Maps}
\label{sec:split-hadamard}

We now identify a smooth global condition that produces rigid coherent
information sheets.  Let \((M,g)\) and \((B,h)\) be finite-dimensional
Hadamard manifolds, and let \(q:M\to B\) be a surjective Riemannian
submersion.  Its vertical and horizontal distributions are
\[
  \mathcal V_x=\ker Dq(x),
  \qquad
  \mathcal H_x=\mathcal V_x^{\perp_g}.
\]

\begin{definition}[Split-Hadamard information map]
\label{def:split-hadamard}
The map \(q\) is split-Hadamard if \(\mathcal H\) is integrable and every
maximal horizontal leaf is complete and totally geodesic.
\end{definition}

This condition rules out holonomy between visible and invisible directions.
It does not require the fibers themselves to be linear or compact.

\begin{theorem}[Global split-Hadamard information theorem]
\label{thm:split-hadamard}
Let \(q:M\to B\) be split-Hadamard.  For every \(x\in M\), let \(L_x\) be
the horizontal leaf through \(x\).  Then:
\begin{enumerate}[label=(\roman*),leftmargin=*]
  \item The restriction \(q|_{L_x}:L_x\to B\) is a global Riemannian
  isometry.  Its inverse
  \[
    s_x:B\to L_x\subset M
  \]
  is the unique horizontal isometric totally geodesic section satisfying
  \(s_x(q(x))=x\).
  \item The map \(q\) is a metric submetry and \(s_x\) is the unique shortest
  completion:
  \begin{equation}
    s_x(b)=\argminop_{q(z)=b}d_M(x,z),
    \qquad
    d_M(x,s_x(b))=d_B(q(x),b).
    \label{eq:split-shortest-lift}
  \end{equation}
  \item Universal radial reduction \eqref{eq:universal-radial-reduction}
  holds.  Every reduced minimizer lifts through \(s_x\); if \(\rho\) is
  strictly increasing on its finite domain, these are all finite-valued full
  minimizers.
  \item Moreau envelopes and proximal maps commute exactly:
  \begin{align}
    e_\tau^M(\phi\circ q)(x)&=e_\tau^B\phi(q(x)),
    \label{eq:split-moreau}\\
    \prox^M_{\tau(\phi\circ q)}(x)
    &=s_x\!\left(\prox^B_{\tau\phi}(q(x))\right)
    \label{eq:split-prox}
  \end{align}
  for proper lower-semicontinuous geodesically convex bounded-below
  \(\phi\).
  \item If \(\phi\in C^1(B)\), then
  \begin{equation}
    \operatorname{grad}_M(\phi\circ q)(z)
    =
    \bigl(Dq(z)|_{\mathcal H_z}\bigr)^{-1}
    \operatorname{grad}_B\phi(q(z)).
    \label{eq:split-gradient}
  \end{equation}
  Consequently, every solution \(y:I\to B\) of
  \(\dot y=-\operatorname{grad}_B\phi(y)\) with \(y(0)=q(x)\) lifts on the
  same existence interval to the solution
  \(z=s_x\circ y\) of
  \(\dot z=-\operatorname{grad}_M(\phi\circ q)(z)\) with \(z(0)=x\).
  If \(\operatorname{grad}_B\phi\) is locally Lipschitz, the corresponding
  maximal solutions are unique.
\end{enumerate}
\end{theorem}

The proof in \cref{app:split-proofs} has two ingredients.  Completeness
extends every horizontal lift of a visible geodesic, and Hadamard uniqueness
makes the restriction to a horizontal leaf globally injective.  Equality in
Riemannian-submersion length contraction forces a shortest ambient geodesic
to be horizontal, yielding uniqueness of the nearest completion.

\begin{corollary}[Rigid coherent quotient]
\label{cor:split-coherent}
The sections \(s_x\) satisfy \eqref{eq:coherent-sections}; hence every
split-Hadamard map is a rigid coherent metric quotient in the sense of
\cref{thm:coherent-reduction}.
\end{corollary}

The abstract theorem is useful only after its hypotheses have been verified.
The next section proves them globally for SPD information compression and, in
addition, resolves the section in closed form.

\section{SPD Information Compression and Canonical Completion}
\label{sec:spd-realization}

Let \(A\in\R^{d\times m}\) have full column rank.  The associated information
compression is
\begin{equation}
  \pi_A:\SPD^d\to\SPD^m,
  \qquad
  \pi_A(P)=A^\top P A.
  \label{eq:spd-compression}
\end{equation}
For \(R,S\in\SPD^m\), write
\begin{equation}
  R\#_tS
  =R^{1/2}(R^{-1/2}SR^{-1/2})^tR^{1/2},
  \qquad 0\le t\le1,
  \label{eq:airm-geodesic}
\end{equation}
for AIRM geodesic interpolation.
For a reference \(P_0\in\SPD^d\), write
\[
  R_0=A^\top P_0A,
  \qquad
  U_0=P_0^{1/2}AR_0^{-1/2},
  \qquad U_0^\top U_0=I_m.
\]
For \(R\in\SPD^m\), define
\begin{align}
  \Psi_{P_0,A}(R)
  &=P_0^{1/2}
  \left[I+U_0(R_0^{-1/2}RR_0^{-1/2}-I)U_0^\top\right]
  P_0^{1/2}
  \nonumber\\
  &=P_0+P_0AR_0^{-1}(R-R_0)R_0^{-1}A^\top P_0.
  \label{eq:canonical-completion}
\end{align}
This expression is invariant under the visible coordinate change
\((A,R)\mapsto(AB,B^\top RB)\), \(B\in GL(m)\).

\begin{proposition}[A reference is necessary]
\label{prop:reference-necessary}
Let \(W\subsetneq E\) and \(R\in\SPD(W)\).  There is no rule depending only
on \((W,R)\) that selects a full \(P\in\SPD(E)\), restricts to \(R\) on
\(W\), and is invariant under every linear automorphism of \(E\) that fixes
\(W\) pointwise.  Canonical ambient completion therefore requires additional
structure such as \(P_0\).
\end{proposition}

\begin{theorem}[Complete SPD information-compression theorem]
\label{thm:spd-compression}
The map \(\pi_A\) and the section \(\Psi_{P_0,A}\) satisfy the following.
\begin{enumerate}[label=(\Alph*),leftmargin=*]
  \item \emph{Contraction and submetry.}  For all \(P,Q\in\SPD^d\),
  \begin{equation}
    \dai(A^\top PA,A^\top QA)\le\dai(P,Q),
    \label{eq:spd-contraction}
  \end{equation}
  and for every \(r\ge0\),
  \begin{equation}
    \pi_A(\overline B_{\AI}(P_0,r))
    =\overline B_{\AI}(R_0,r).
    \label{eq:spd-ball-submetry}
  \end{equation}
  \item \emph{Unique canonical completion.}  For every \(R\in\SPD^m\),
  \begin{equation}
    \boxed{
    \Psi_{P_0,A}(R)
    =\argminop_{A^\top PA=R}\dai(P_0,P),
    \qquad
    \dai(P_0,\Psi_{P_0,A}(R))=\dai(R_0,R).}
    \label{eq:spd-shortest-completion}
  \end{equation}
  The minimizer is unique.
  \item \emph{Horizontal geometry.}  The section \(\Psi_{P_0,A}\) is an
  isometric totally geodesic embedding.  The map \(\pi_A\) is a Riemannian
  submersion; at \(P\), the inverse of
  \(D\pi_A(P)|_{\mathcal H_P}\) is
  \begin{equation}
    h\longmapsto D\Psi_{P,A}(A^\top PA)[h].
    \label{eq:spd-horizontal-lift}
  \end{equation}
  Hence \(\pi_A\) is split-Hadamard and visible AIRM geodesics have unique
  horizontal ambient lifts.
  \item \emph{Invisible-fiber minimax geometry.}  If \(m<d\), the fiber
  \(\mathcal F_A(R)=\{P:A^\top PA=R\}\) has infinite AIRM diameter.  With
  \(P_\star=\Psi_{P_0,A}(R)\) and
  \[
    \mathcal F_A(R;B)
    =\mathcal F_A(R)\cap\overline B_{\AI}(P_\star,B),
  \]
  one has
  \begin{equation}
    \inf_{Q\in\SPD^d}\sup_{P\in\mathcal F_A(R;B)}\dai(Q,P)=B,
    \label{eq:spd-minimax}
  \end{equation}
  and \(P_\star\) is the unique minimax center.
  \item \emph{Exact radial decision reduction.}  Let \(\mathcal A\) be any
  family of admissible \((P_0,A)\), allowing different visible dimensions,
  let \(\rho\) be nondecreasing, and let \(\varphi\) be any extended-real
  visible functional.  Then
  \begin{align}
    &\inf_{\substack{(P_0,A)\in\mathcal A\\P\in\SPD^d}}
    \left\{\rho(\dai(P_0,P))+\varphi(P_0,A,A^\top PA)\right\}
    \nonumber\\
    &\quad=
    \inf_{\substack{(P_0,A)\in\mathcal A\\R\in\SPD^m}}
    \left\{\rho(\dai(A^\top P_0A,R))+\varphi(P_0,A,R)\right\}.
    \label{eq:spd-decision-reduction}
  \end{align}
  Every reduced minimizer has the canonical full lift
  \(P=\Psi_{P_0,A}(R)\).  If \(\rho\) is strictly increasing on its finite
  domain, every finite-valued full minimizer is canonical.
\end{enumerate}
\end{theorem}

The set \(\mathcal F_A(R;B)\) is a local AIRM ambiguity model centered at
the reference-selected completion \(P_\star\).  Thus
\eqref{eq:spd-minimax} is a radius-conditioned nonidentifiability statement:
it does not define a reference-free uncertainty set or a reference-free
minimax center.

\begin{corollary}[Visible update quotient and invisible drift]
\label{cor:visible-update}
Let \(\alpha=Ac\ne0\) be a visible covector and let
\(u_P=-P\alpha\).  For every \(P\) satisfying \(A^\top PA=R\),
\begin{equation}
  A^\top u_P=-Rc.
  \label{eq:visible-update-quotient}
\end{equation}
Thus the class of the update modulo \(\ker A^\top\) is fixed.  Conversely,
for every \(n\in\ker A^\top\), there exists a feasible \(P\) such that
\begin{equation}
  u_P=u_{\Psi_{P_0,A}(R)}+n.
  \label{eq:arbitrary-invisible-drift}
\end{equation}
The canonical completion is the unique AIRM-nearest representative and has no
invisible component in the \(P_0\)-orthogonal splitting.
\end{corollary}

\begin{remark}[What is selected]
The visible target \(R\) is an input to the theorem.  The result selects the
least AIRM deformation of \(P_0\) compatible with that target.  It does not
assert that a first-order oracle uniquely determines \(R\), nor that the
reference \(P_0\) is reference-free.
\end{remark}

The same completion also has an information-theoretic characterization.
For Gaussians it can equivalently be read from the KL chain rule as retaining
the reference conditional law on invisible coordinates.  We include it as a
consequence and alternate characterization of the AIRM completion, rather
than as an independent divergence-projection novelty claim.

\begin{corollary}[Gaussian KL completion]
\label{cor:gaussian-kl}
Viewing \(P\) as the covariance of a zero-mean Gaussian, the problem
\begin{equation}
  \min_{A^\top PA=R}
  D_{\mathrm{KL}}\!\left(N(0,P)\,\|\,N(0,P_0)\right)
  \label{eq:gaussian-kl-problem}
\end{equation}
has the unique minimizer \(\Psi_{P_0,A}(R)\).
\end{corollary}

For a fixed information channel, exact decision reduction yields an explicit
regularized update.  Let \(R_0=A^\top P_0A\), let \(\Rhat\in\SPD^m\) be a
positive data-visible target, and let \(0<\tau<1\).

\begin{corollary}[Prior--data geodesic completion]
\label{cor:prior-data}
The full-SPD objective
\begin{equation}
  (1-\tau)\dai(P_0,P)^2
  +\tau\dai(A^\top PA,\Rhat)^2
  \label{eq:prior-data-objective}
\end{equation}
has the unique minimizer
\begin{equation}
  \boxed{
  P_\tau=\Psi_{P_0,A}(R_0\#_\tau\Rhat).}
  \label{eq:prior-data-solution}
\end{equation}
If \(D=\dai(R_0,\Rhat)\), then
\[
  \dai(P_0,P_\tau)=\tau D,
  \qquad
  \dai(A^\top P_\tau A,\Rhat)=(1-\tau)D,
\]
and the minimum objective value is \(\tau(1-\tau)D^2\).  Perturbing the
visible target from \(R_1\) to \(R_2\) changes the completed solution by at
most \(\tau\dai(R_1,R_2)\).
\end{corollary}

The proofs of
\cref{prop:reference-necessary,thm:spd-compression,cor:visible-update,cor:gaussian-kl,cor:prior-data}
are in
\cref{app:spd-proofs}.  Their common mechanism is that orthogonal compression
contracts the whitened Frobenius norm, while the section
\eqref{eq:canonical-completion} attains equality.

\section{Gauge, Rank Strata, and Moving Information}
\label{sec:stratified-geometry}

The previous section fixes \(A\).  We now let the information channel move.
Relative to \(P_0\), the pair \((A,R)\) has a change-of-visible-coordinates
gauge
\[
  (A,R)\cdot B=(AB,B^\top RB),
  \qquad B\in GL(m).
\]
First define the total, generally redundant descriptor space
\begin{equation}
  \widetilde{\mathfrak D}_m(P_0)
  =
  \left\{(A,R):\rank A=m,\ R\in\SPD^m\right\}/GL(m),
  \label{eq:total-descriptor-space}
\end{equation}
and the everywhere-defined completion parameterization
\begin{equation}
  \widetilde\Gamma_{m,P_0}([A,R])=\Psi_{P_0,A}(R).
  \label{eq:total-completion-map}
\end{equation}
It is redundant when some visible modes equal their reference values.  Its
regular subset is
\begin{equation}
  \mathfrak D_m^{\mathrm{reg}}(P_0)
  =
  \left\{(A,R):\rank A=m,\ R\in\SPD^m,
  \ \det(R-A^\top P_0A)\ne0\right\}/GL(m),
  \label{eq:descriptor-space}
\end{equation}
and the rank stratum
\begin{equation}
  \mathfrak S_m(P_0)
  =\{P\in\SPD^d:\rank(P-P_0)=m\}.
  \label{eq:rank-stratum}
\end{equation}

\begin{theorem}[Completion bundle and complete rank stratification]
\label{thm:rank-stratification}
The map
\begin{equation}
  \Gamma_{m,P_0}([A,R])=\Psi_{P_0,A}(R)
  \label{eq:stratum-map}
\end{equation}
is a diffeomorphism from
\(\mathfrak D_m^{\mathrm{reg}}(P_0)\) onto
\(\mathfrak S_m(P_0)\).  Its inverse descriptor is intrinsic:
\begin{equation}
  \range(A)=P_0^{-1}\range(P-P_0),
  \qquad
  R=A^\top PA,
  \label{eq:intrinsic-descriptor}
\end{equation}
up to the stated \(GL(m)\) action.  The strata exhaust and close according to
\begin{equation}
  \SPD^d=\bigsqcup_{m=0}^d\mathfrak S_m(P_0),
  \qquad
  \overline{\mathfrak S_m(P_0)}^{\,\SPD^d}
  =\bigcup_{r=0}^m\mathfrak S_r(P_0).
  \label{eq:strata-closure}
\end{equation}
\end{theorem}

The pullback AIRM seminorm on the total descriptor space resolves the cost of
changing the visible target and the visible subspace.  It becomes a genuine
quotient metric on the regular subset.  Normalize a representative by
\(A^\top P_0A=I_m\), use
the horizontal gauge \(A^\top P_0\dot A=0\), choose \(N\) so that
\([P_0^{1/2}A,N]\) is orthogonal, and put
\(K=N^\top P_0^{1/2}\dot A\).

\begin{theorem}[Closed-form completion-pair metric and regular quotient geometry]
\label{thm:stratified-line-element}
At any normalized representative of
\(\widetilde{\mathfrak D}_m(P_0)\), represent two horizontal quotient
directions by
\[
  \xi_i=(H_i,K_i),
  \qquad H_i\in\Sym^m,
  \qquad K_i=N^\top P_0^{1/2}\dot A_i,
  \qquad i\in\{1,2\}.
\]
Then the pullback AIRM pair kernel is
\begin{equation}
  \boxed{
  \begin{aligned}
  \mathfrak G_R^{\mathrm{comp}}(\xi_1,\xi_2)
  &:={}
  \left\langle
  D\widetilde\Gamma_{m,P_0}[\xi_1],
  D\widetilde\Gamma_{m,P_0}[\xi_2]
  \right\rangle_{\widetilde\Gamma_{m,P_0},\AI}
  \\
  &={}
  \tr(R^{-1}H_1R^{-1}H_2)
  +2\tr\!\left(K_1\Xiop(R)K_2^\top\right),
  \end{aligned}}
  \label{eq:completion-pair-kernel}
\end{equation}
where
\begin{equation}
  \Xiop(R)=(R-I)R^{-1}(R-I)=R+R^{-1}-2I\succeq0.
  \label{eq:mismatch-weight}
\end{equation}
In particular, for \(\xi=(\dot R,K)\),
\begin{equation}
  \boxed{
  \norm{D\widetilde\Gamma_{m,P_0}[\dot A,\dot R]}_{\widetilde\Gamma_{m,P_0},\AI}^2
  =\mathfrak G_R^{\mathrm{comp}}(\xi,\xi)
  =
  \fro{R^{-1/2}\dot R R^{-1/2}}^2
  +2\tr\!\left(K\Xiop(R)K^\top\right),}
  \label{eq:stratified-line-element}
\end{equation}
so visible-target and subspace-motion blocks are orthogonal.

More explicitly, define the completion feature
\begin{equation}
  \Phi_R(\xi)
  =
  \left(
  R^{-1/2}HR^{-1/2},
  \sqrt{2}\,K(R-I)R^{-1/2}
  \right).
  \label{eq:completion-feature}
\end{equation}
Then
\(
\mathfrak G_R^{\mathrm{comp}}(\xi_i,\xi_j)
=\langle\Phi_R(\xi_i),\Phi_R(\xi_j)\rangle_F
\)
in the direct-sum Frobenius space.  Hence for any directions
\(\xi_1,\ldots,\xi_k\), the pair matrix
\begin{equation}
  \mathbf G
  =\left[\mathfrak G_R^{\mathrm{comp}}(\xi_i,\xi_j)\right]_{i,j=1}^k
  \succeq0,
  \qquad
  \rank\mathbf G
  =\dim\operatorname{span}\{\Phi_R(\xi_i):1\le i\le k\}.
  \label{eq:completion-gram-rank}
\end{equation}
Moreover,
\begin{equation}
  \mathfrak G_R^{\mathrm{comp}}(\xi_i,\xi_j)=0
  \quad\Longleftrightarrow\quad
  \Phi_R(\xi_i)\perp\Phi_R(\xi_j).
  \label{eq:completion-pair-orthogonality}
\end{equation}
The exact horizontal differential kernel of the total parameterization is
\begin{equation}
  \ker D\widetilde\Gamma_{m,P_0}
  =\{(\dot A,0):K(R-I)=0\}.
  \label{eq:stratified-kernel}
\end{equation}
Consequently the pullback is nondegenerate precisely on the regular subset
\(\det(R-I)\ne0\), where
\(\widetilde\Gamma_{m,P_0}=\Gamma_{m,P_0}\) is the diffeomorphism of
\cref{thm:rank-stratification}.  At singular descriptors, rotation of a
reference-valued visible mode is unidentifiable from the completed matrix.
This is a singularity of the redundant descriptor map and its pullback
seminorm; the ambient SPD manifold and its AIRM remain smooth and
nondegenerate.
If
\(\operatorname{dist}(\operatorname{spec}(R),1)\ge\zeta>0\) and
\(0<a\le\lambda(R)\le b\), then subspace motion is quantitatively controlled:
\begin{align}
  \fro{R^{-1/2}\dot R R^{-1/2}}^2
  +\frac{2\zeta^2}{b}\fro K^2
  &\le \norm{\dot P}_{P,\AI}^2
  \nonumber\\
  &\le
  \fro{R^{-1/2}\dot R R^{-1/2}}^2
  +2\max_{x\in[a,b]}\frac{(x-1)^2}{x}\fro K^2.
  \label{eq:stratified-conditioning}
\end{align}
\end{theorem}

\begin{corollary}[Reduced completion value and completed motion]
\label{cor:reduced-completion-value}
At fixed \(P_0\), in the normalized chart \(A^\top P_0A=I\), the reduced
minimum-distance value is
\begin{equation}
  \overline E(A,R)
  :=\min_{A^\top PA=R}\frac12\dai(P_0,P)^2
  =\frac12\dai(I,R)^2.
  \label{eq:reduced-completion-energy}
\end{equation}
It contains no independent information-subspace coordinate.  Consequently,
holding \(R\) fixed, all first and mixed parameter derivatives generated by
horizontal variations of \(A\) vanish.  Exact completion has already removed
the need for an implicit vertical stationarity solve: the minimizing ambient
matrix is the explicit section \(P=\Psi_{P_0,A}(R)\).  The pullback pair
metric \eqref{eq:completion-pair-kernel} is a different second-order object:
it measures motion of this minimizing matrix rather than curvature of the
reduced value.
\end{corollary}

\begin{remark}[Reduced value versus completion motion]
\label{rem:value-versus-motion}
The distinction is already visible for \(d=2\), \(m=1\), \(P_0=I\), and a
scalar target \(r>0\), \(r\ne1\).  With
\(u(\theta)=(\cos\theta,\sin\theta)^\top\), the canonical completion is
\[
  P_\star(\theta)=I+(r-1)u(\theta)u(\theta)^\top.
\]
Its reduced minimum-distance value is the constant
\(\tfrac12\log^2r\), whereas
\[
  \norm{\dot P_\star(0)}_{P_\star(0),\AI}^2
  =\frac{2(r-1)^2}{r}>0.
\]
Thus vanishing subspace derivatives of the optimized value coexist with
nonzero metric motion of its canonical minimizer.  A frozen--relaxed
second-variation decomposition may likewise contain nonzero direct and
Schur terms that cancel; their split depends on the chosen trivialization of
the moving constraint.  The completion kernel itself is analytic for every
\(R\in\SPD^m\), including repeated eigenvalues, and its degeneracy is exactly
the descriptor non-identifiability in \eqref{eq:stratified-kernel}.
\end{remark}

The same identity has a gauge-covariant moving-reference form.  Let
\(P_{0,t}=G_tG_t^\top\), define
\[
  R_{0,t}=A_t^\top P_{0,t}A_t,
  \quad
  U_t=G_t^\top A_tR_{0,t}^{-1/2},
  \quad
  \bar R_t=R_{0,t}^{-1/2}R_tR_{0,t}^{-1/2},
\]
\[
  Q_t=I+U_t(\bar R_t-I)U_t^\top,
  \qquad
  \Omega_t=G_t^{-1}\dot G_t,
  \qquad
  \mathfrak D_tQ_t=\dot Q_t+\Omega_tQ_t+Q_t\Omega_t^\top.
\]

\begin{proposition}[Covariant moving-reference identity]
\label{prop:moving-reference}
For \(P_t=\Psi_{P_{0,t},A_t}(R_t)=G_tQ_tG_t^\top\),
\begin{equation}
  \boxed{
  \norm{\dot P_t}_{P_t,\AI}^2
  =\fro{Q_t^{-1/2}(\mathfrak D_tQ_t)Q_t^{-1/2}}^2.}
  \label{eq:moving-reference}
\end{equation}
Under \(G_t'=G_tO_t\) with \(O_t\) orthogonal,
\[
  Q_t'=O_t^\top Q_tO_t,
  \qquad
  \mathfrak D_t'Q_t'=O_t^\top(\mathfrak D_tQ_t)O_t.
\]
Hence \eqref{eq:moving-reference} is independent of the moving frame.
\end{proposition}

For a fixed reference and a fixed-rank curve, integrating
\eqref{eq:stratified-line-element} gives the exact energy decomposition
\begin{align}
  \int_0^T\norm{\dot P_t}_{P_t,\AI}^2\,\dd t
  &=\int_0^T\fro{R_t^{-1/2}\dot R_tR_t^{-1/2}}^2\,\dd t
  \nonumber\\
  &\quad+2\int_0^T\tr(K_t\Xiop(R_t)K_t^\top)\,\dd t.
  \label{eq:integrated-energy}
\end{align}

\begin{corollary}[Boundary of continuous rank loss]
\label{cor:rank-loss}
Suppose
\(P_t=P_0^{1/2}[I+U_t(R_t-I)U_t^\top]P_0^{1/2}\) has active rank \(m\)
for \(t<t_0\) and converges to a point of rank at most \(m-r\) relative to
\(P_0\).  Then at least \(r\) singular values of \(R_t-I\) converge to zero.
Thus continuous rank loss can occur only through visible modes approaching the
reference geometry.
\end{corollary}

The proofs are collected in \cref{app:stratified-proofs}.  The line element is
a kinematic identity on the completion bundle.  A learning or control law
would require an additional evolution rule for \((P_0,A,R)\).

\section{Structured Expressivity Certificates}
\label{sec:structured-expressivity}

The full-SPD channel can realize every \(R\in\SPD^m\).  Implementable
optimizers often restrict \(P\) to a coordinate, block, Kronecker, or other
structured family.  For a family \(\Fcal\subset\SPD^d\) equipped with a
within-family distance \(d_\Fcal\), define the visible restricted complexity
\begin{equation}
  D_{\Fcal}^{\mathrm{vis}}(A,R;P_0)
  =
  \inf_{\substack{P\in\Fcal\\A^\top PA=R}}d_\Fcal(P_0,P),
  \label{eq:visible-restricted-complexity}
\end{equation}
with value \(+\infty\) when no structured completion exists.  This section
characterizes finiteness for diagonal and block families.

\subsection{Diagonal geometry}

Let \(a_j^\top\) be the \(j\)-th row of \(A\in\R^{d\times m}\) and set
\(G_j=a_ja_j^\top\in\PSD^m\).  A positive diagonal cometric
\(D(p)=\diag(p_1,\ldots,p_d)\) has visible image
\begin{equation}
  A^\top D(p)A=\sum_{j=1}^d p_jG_j.
  \label{eq:diagonal-image}
\end{equation}
Let \(\Kcal_{\mathrm{diag}}=\cone(G_1,\ldots,G_d)\).

\begin{theorem}[Diagonal expressivity and facial alternative]
\label{thm:diagonal-expressivity}
Let \(D_0=\diag(p_1^0,\ldots,p_d^0)\) with \(p_j^0>0\), and define
\begin{equation}
  D_{\mathrm{diag}}^{\mathrm{vis}}
  =
  \inf_{\substack{p_j>0\\\sum_jp_jG_j=R}}
  \left(\sum_{j=1}^d\log^2\frac{p_j}{p_j^0}\right)^{1/2}.
  \label{eq:diagonal-complexity}
\end{equation}
Then \(D_{\mathrm{diag}}^{\mathrm{vis}}<\infty\) if and only if
\begin{equation}
  R\in\ri\Kcal_{\mathrm{diag}},
  \label{eq:diagonal-ri}
\end{equation}
equivalently, \(R=\sum_jp_jG_j\) for some \(p_j>0\).  When finite, the
infimum is attained.  If the \(G_j\) are linearly independent, the feasible
diagonal cometric is unique.

There are two distinct failure certificates.
\begin{enumerate}[label=(\roman*),leftmargin=*]
  \item If \(R\notin\Kcal_{\mathrm{diag}}\), there exists symmetric \(Y\)
  such that
  \[
    \tr(YR)<0,
    \qquad
    \tr(YG_j)\ge0\quad\forall j.
  \]
  \item If
  \(R\in\Kcal_{\mathrm{diag}}\setminus\ri\Kcal_{\mathrm{diag}}\), there
  exists symmetric \(Y\) such that
  \[
    \tr(YR)=0,
    \qquad
    \tr(YG_j)\ge0\quad\forall j,
    \qquad
    \tr(YG_{j_\star})>0\quad\text{for some }j_\star.
  \]
\end{enumerate}
The first strictly separates an infeasible target; the second exposes the
proper face containing a boundary target that requires some diagonal weights
to vanish.
\end{theorem}

\begin{example}[An end-to-end visible training channel]
\label{ex:worked-training-channel}
Let the rows of \(A\in\R^{3\times2}\) be
\[
  a_1^\top=(1,0),
  \qquad
  a_2^\top=(0,1),
  \qquad
  a_3^\top=(1,1).
\]
Take \(P_0=I_3\) and the visible target \(R=I_2\).  Since
\[
  R_0=A^\top A=\begin{pmatrix}2&1\\1&2\end{pmatrix},
\]
the full-SPD canonical completion is
\[
  P_\star=\Psi_{I_3,A}(I_2)
  =\frac19
  \begin{pmatrix}
    8&-1&-2\\
    -1&8&-2\\
    -2&-2&5
  \end{pmatrix},
  \qquad A^\top P_\star A=I_2.
\]
Thus the visible request has a unique nearest full-SPD realization.  Within
the diagonal family, however, it requires the semidefinite coefficient vector
\(p=(1,1,0)\); no strictly positive diagonal geometry realizes it.  The
matrix
\[
  Y=\begin{pmatrix}0&1\\1&0\end{pmatrix}
\]
is a facial certificate: \(\tr(YR)=0\),
\(\tr(YG_1)=\tr(YG_2)=0\), and \(\tr(YG_3)=2\), so it is a nontrivial
functional on \(\operatorname{span}\Kcal_{\mathrm{diag}}\).

For the visible covector \(\alpha=Ae_1=(1,0,1)^\top\), the canonical update
is
\[
  u_\star=-P_\star\alpha
  =\left(-\frac23,\frac13,-\frac13\right)^\top,
  \qquad A^\top u_\star=-e_1.
\]
Every feasible full geometry has this same visible update, while its invisible
component may differ by any vector in
\(\ker A^\top=\operatorname{span}\{(1,1,-1)^\top\}\).  The example
therefore displays, in one calculation, exact visible reduction, canonical
full completion, invisible drift, and a structured-family boundary
certificate.
\end{example}

\subsection{Block geometry}

Partition the rows of \(A\) into blocks \(A_b\).  A block cometric has the
form \(P=\blockdiag(P_1,\ldots,P_B)\), \(P_b\succ0\), and visible image
\begin{equation}
  A^\top PA=\sum_bA_b^\top P_bA_b.
  \label{eq:block-image}
\end{equation}
Define
\[
  \Kcal_{\mathrm{block}}
  =\left\{\sum_bA_b^\top P_bA_b:P_b\succeq0\right\}.
\]

\begin{theorem}[Block expressivity]
\label{thm:block-expressivity}
For a block reference \(P_0=\blockdiag(P_1^0,\ldots,P_B^0)\), the quantity
\begin{equation}
  D_{\mathrm{block}}^{\mathrm{vis}}
  =
  \inf_{\substack{P_b\succ0\\\sum_bA_b^\top P_bA_b=R}}
  \left(\sum_b\dai(P_b^0,P_b)^2\right)^{1/2}
  \label{eq:block-complexity}
\end{equation}
is finite if and only if
\(R\in\ri\Kcal_{\mathrm{block}}\), equivalently if positive definite
blocks realize \(R\).  When finite, the infimum is attained.

The cone \(\Kcal_{\mathrm{block}}\) is closed.  If
\(R\notin\Kcal_{\mathrm{block}}\), there is symmetric \(Y\)
such that
\[
  \tr(YR)<0,
  \qquad
  A_bYA_b^\top\succeq0\quad\forall b.
\]
If
\(R\in\Kcal_{\mathrm{block}}\setminus
\ri\Kcal_{\mathrm{block}}\), there is a facial certificate with
\[
  \tr(YR)=0,
  \qquad
  A_bYA_b^\top\succeq0\quad\forall b,
  \qquad
  A_{b_\star}YA_{b_\star}^\top\ne0\quad\text{for some }b_\star.
\]
\end{theorem}

The proofs in \cref{app:expressivity-proofs} use the identity that a linear
map sends the relative interior of a finite-dimensional convex cone onto the
relative interior of its image.  The distinction between strict separation
and facial exposure is essential: targets on a proper face are realizable by
semidefinite structured geometries while remaining unreachable by strictly
positive optimizer blocks.  In both facial alternatives, nontriviality is
measured on the linear span of the image cone; ambient nonzero \(Y\) alone is
insufficient because it may annihilate the entire cone.

The certificates have direct finite-dimensional feasibility forms.  For the
diagonal boundary case, scaling permits the normalization
\(\sum_j\tr(YG_j)=1\), leaving linear equalities and inequalities in \(Y\);
strict infeasibility can be normalized by \(\tr(YR)\le-1\).  For block
geometry, the corresponding constraints
\(A_bYA_b^\top\succeq0\) together with
\(\sum_b\tr(A_bYA_b^\top)=1\) form an SDP feasibility problem.  These
formulations make the certificates computable in principle; algorithmic
complexity and numerical tolerances are outside the present characterization.

Kronecker, low-rank, and tensor families require separate image geometry and
are not covered by the diagonal/block theorems.

\section{Recovery of Visible and Completed Geometry}
\label{sec:statistical-recovery}

The isometric fixed-channel section transfers visible SPD error exactly.  When
the information subspace is estimated, the stratified line element adds a
second, mismatch-weighted Grassmann term.

All results below are stated after congruence whitening the reference to
\(P_0=I\); AIRM congruence invariance transports them to an arbitrary fixed
reference.  In these coordinates, for an orthonormal frame
\(U\in\St(d,m)\), write
\begin{equation}
  \Cpl(U,R)=I+U(R-I)U^\top.
  \label{eq:normalized-completion}
\end{equation}

Two observation regimes must be distinguished.  The exact visible space in
\cref{sec:introduction} is the span of the observed covectors.  The
known-subspace result below conditions on any such fixed space.  The
unknown-subspace result instead adopts a population second-moment model and
\emph{defines} its target space as the leading \(m\)-dimensional eigenspace
of \(\Sigma\).  This spectral target agrees with the population covector
support under a rank-\(m\) population model; a finite observed span equals it
only when the sampled covectors span that support.  With nonzero spectral
tails it instead represents a chosen principal approximation.  No automatic
identification of the two regimes is assumed.

For \(W,\widetilde W\in\Gr(m,d)\), let \(\dgr(W,\widetilde W)\) be the
Euclidean norm of their principal-angle vector.  Define
\begin{equation}
  \chi(R)=\opnorm{\Xiop(R)}^{1/2}
  =\max_{\lambda\in\operatorname{spec}(R)}
  \frac{|\lambda-1|}{\sqrt\lambda}.
  \label{eq:chi}
\end{equation}

\begin{theorem}[Joint deterministic perturbation]
\label{thm:joint-perturbation}
Let \(U,V\) be endpoint frames connected by a horizontal minimizing
Grassmann geodesic, and express \(R,S\in\SPD^m\) in the parallel endpoint
frames.  With reference \(I\),
\begin{equation}
  \boxed{
  \dai(\Cpl(U,R),\Cpl(V,S))
  \le
  \dai(R,S)
  +\sqrt2\min\{\chi(R),\chi(S)\}\dgr(W,\widetilde W).}
  \label{eq:joint-perturbation}
\end{equation}
The first term is visible geometry error.  The second is subspace error, and
it vanishes for reference-valued visible geometry.
\end{theorem}

For a known fixed visible subspace, a direct concentration theorem retains
the intrinsic dimension.  In an orthonormal coordinate frame of \(W\), let
\(X_1,\ldots,X_n\in\R^m\) be i.i.d. copies of \(X\), define the uncentered
second moment \(R=\E[XX^\top]\succ0\), assume
\(\norm{R^{-1/2}X}_2\le L\) almost surely, and set
\(\Rhat=n^{-1}\sum_iX_iX_i^\top\).

\begin{theorem}[Known-subspace finite-sample recovery]
\label{thm:known-subspace-recovery}
For \(0<\varepsilon\le1\), if
\begin{equation}
  n\ge
  \frac{8L^4}{\varepsilon^2}
  \left(m\log9+\log\frac{2}{\delta}\right),
  \label{eq:known-subspace-sample-size}
\end{equation}
then with probability at least \(1-\delta\),
\begin{equation}
  e^{-\varepsilon}R\preceq\Rhat\preceq e^\varepsilon R,
  \qquad
  \dai(\Cpl(U,R),\Cpl(U,\Rhat))\le\sqrt m\,\varepsilon.
  \label{eq:known-subspace-bound}
\end{equation}
\end{theorem}

This self-contained net--Hoeffding bound is deliberately conservative.  In
fact
\[
  m=\E\norm{R^{-1/2}X}_2^2\le L^2,
\]
so the displayed sufficient sample size can have order
\(m^3/\varepsilon^2\) when only this envelope is retained.  Matrix
Chernoff or Bernstein inequalities can replace the net argument and often
improve the dependence to a logarithmic dimension or an effective-rank
factor under their corresponding variance or tail assumptions
\citep{tropp2015matrix}.  We use the elementary bound for transparency and
do not claim an optimal statistical rate.

We next allow the visible eigenspace to be unknown.  Let \(1\le m<d\) and
\(\Sigma\in\PSD^d\) have eigenvalues
\(\lambda_1\ge\cdots\ge\lambda_d\), assume
\(\gamma=\lambda_m-\lambda_{m+1}>0\), and let \(U\) span its leading
\(m\)-dimensional eigenspace.  Put
\[
  R=U^\top\Sigma U,
  \qquad
  P=\Cpl(U,R).
\]
For an estimator \(\Sigmahat\), let \(V\) span its leading eigenspace,
\(\Rhat=V^\top\Sigmahat V\), and \(\Phat=\Cpl(V,\Rhat)\).

\begin{theorem}[Unknown-subspace completed-geometry recovery]
\label{thm:unknown-subspace-recovery}
Let
\[
  \delta=\opnorm{\Sigmahat-\Sigma}
  <\min\{\gamma/2,\lambda_m\},
  \qquad
  s_\delta=\frac{\delta}{\gamma-\delta},
\]
\[
  \varepsilon_{\mathrm{vis}}
  =\delta+2\sqrt2\lambda_1s_\delta,
  \qquad
  \eta_\delta=\frac{\varepsilon_{\mathrm{vis}}}{\lambda_m}<1,
\]
and
\[
  \omega_\delta
  =\max_{x\in[\lambda_m-\delta,\lambda_1+\delta]}
  \frac{|x-1|}{\sqrt x}.
\]
Then
\begin{equation}
  \boxed{
  \dai(P,\Phat)
  \le
  \sqrt m\left[-\log(1-\eta_\delta)
  +\sqrt2\,\omega_\delta\arcsin(s_\delta)\right].}
  \label{eq:unknown-subspace-bound}
\end{equation}
The first term controls the aligned visible second moment; the second is the
mismatch-weighted rotation of the estimated eigenspace.
\end{theorem}

\begin{corollary}[A self-contained sample bound]
\label{cor:unknown-subspace-sample}
Let \(X_1,\ldots,X_n\in\R^d\) be i.i.d. copies of \(X\), with uncentered
second moment \(\Sigma=\E[XX^\top]\), \(\norm X_2\le L\) almost surely,
and
\(\Sigmahat=n^{-1}\sum_iX_iX_i^\top\).  Define
\begin{equation}
  \delta_n(\alpha)
  =L^2\sqrt{\frac{2(d\log9+\log(2/\alpha))}{n}}.
  \label{eq:ambient-delta}
\end{equation}
With probability at least \(1-\alpha\),
\(\opnorm{\Sigmahat-\Sigma}\le\delta_n(\alpha)\).  If
\(\delta_n(\alpha)\) satisfies the hypotheses of
\cref{thm:unknown-subspace-recovery}, then
\eqref{eq:unknown-subspace-bound} holds with
\(\delta=\delta_n(\alpha)\).
\end{corollary}

The ambient dimension in \eqref{eq:ambient-delta} is the price of estimating
the subspace with a self-contained bounded-i.i.d. net argument.  Matrix
concentration or effective-rank assumptions may improve this factor, but an
eigengap or another identifiable subspace condition remains necessary.  The
reference \(I\) here is the whitened reference geometry, and \(\Sigma\) is an
uncentered second moment; it coincides with covariance only when the mean is
zero.  Proofs are in
\cref{app:recovery-proofs}.

\section{Discussion}
\label{sec:discussion}

\paragraph{What the theory establishes.}
Exact information reduction has a precise hierarchy.  At the metric level,
submetry is necessary and sufficient for universal radial-visible reduction.
Coherent sections select reusable information sheets.  Split-Hadamard
geometry turns those sheets into global horizontal submanifolds.  SPD
compression realizes the entire hierarchy in closed form and adds a rank
quotient, a closed-form completion-pair metric, and computable structured
certificates.

\paragraph{Visibility and reference.}
A visible target does not determine an ambient geometry.  The invisible fiber
has infinite diameter, and even its radius-bounded minimax problem has risk
equal to the radius.  The canonical completion is therefore meaningful only
relative to a declared \(P_0\).  Its role is precise: preserve the reference
on directions not forced to change by the visible target while realizing the
visible target with minimum AIRM deformation.

\paragraph{Moving information.}
The mismatch weight \(\Xiop(R)=R+R^{-1}-2I\) explains when subspace motion is
observable in completed geometry.  Rotating a mode whose visible geometry
equals the reference changes no completed matrix.  As a mode separates from
the reference, the same rotation acquires positive cost.  Continuous rank
changes can therefore occur only through reference-valued modes.  These are
kinematic statements; choosing a trajectory requires an additional modeling,
optimization, or control principle.  The bilinear kernel also separates two
second-order objects.  The reduced minimum-distance value is independent of
the normalized subspace coordinate at fixed \(R\), while the minimizing
completed matrix generally moves with that coordinate.  Its motion is
measured by the positive-semidefinite Gram form
\eqref{eq:completion-pair-kernel}.  Reduced-value curvature and pullback
metric geometry therefore answer different questions.

\paragraph{Structured optimizers.}
Diagonal and block geometry are restrictions of the same visible map
\(P\mapsto A^\top PA\).  Their relative-interior tests distinguish three
scientifically different outcomes: strict representability, boundary
representability requiring degeneracy, and infeasibility even after closure.
A single residual value cannot make these distinctions; the conic certificates
can.

\paragraph{Statistical interpretation.}
Known-subspace recovery is intrinsically \(m\)-dimensional.  Estimating the
subspace introduces an ambient second-moment problem and an eigengap.  The exact
line element then converts principal-angle error into completed AIRM error
with a geometry-dependent weight.  The resulting bound separates statistical
uncertainty in \(R\) from uncertainty in the information channel itself.  The
self-contained rates use bounded i.i.d. observations and net--Hoeffding
arguments; they establish consistency in the stated regime and are not
rate-optimal claims.

\paragraph{Boundaries.}
The abstract theorem characterizes exactness but does not make an arbitrary
observation map a submetry.  Nonlinear, infinite-dimensional, multi-view, or
non-AIRM models require separate verification.  The SPD pair metric is
specific to full-column-rank linear compression and is analytic across
repeated eigenvalues.  The structured certificate
theorems cover diagonal and block families; Kronecker, low-rank, and tensor
families have different image geometry.  Finally, the paper assumes a
positive visible target.  Indefinite curvature must be damped, projected, or
otherwise converted before this theory applies.

\paragraph{Downstream use.}
The paper supplies a geometric foundation for later questions about metric
evolution, constrained metric paths, and intervention-based optimizer audits.
Those questions may use the visible state \((A,R)\), the canonical completion
\(\Psi_{P_0,A}(R)\), and the completion-pair metric derived here.  Their
dynamical or causal assumptions are intentionally outside the present paper.

\section{Reproducibility and Ethics Statements}
\label{sec:statements}

\paragraph{Reproducibility.}
This is a theory paper.  Every claim is stated as a definition, theorem,
corollary, or explicitly bounded interpretation.  Complete proofs are given
in the appendices.  The finite-sample results specify the boundedness,
eigengap, confidence, and dimension dependence used by each conclusion.  No
empirical optimizer-performance claim is made.

\paragraph{Ethics and broader impact.}
The paper introduces no dataset, human-subject experiment, or deployed
decision system.  Its main interpretive risk is to treat a visible geometric
certificate as a complete explanation of an optimizer.  The theory identifies
what a declared information channel constrains; memory, noise, higher-order
oracles, discretization, and other mechanisms can produce additional motion.

\bibliographystyle{tmlr}
\bibliography{references}

\begin{thebibliography}{32}
\providecommand{\natexlab}[1]{#1}
\providecommand{\url}[1]{\texttt{#1}}
\expandafter\ifx\csname urlstyle\endcsname\relax
  \providecommand{\doi}[1]{doi: #1}\else
  \providecommand{\doi}{doi: \begingroup \urlstyle{rm}\Url}\fi

\bibitem[Amari(1998)]{amari1998natural}
Shun-ichi Amari.
\newblock Natural gradient works efficiently in learning.
\newblock \emph{Neural Computation}, 10\penalty0 (2):\penalty0 251--276, 1998.
\newblock \doi{10.1162/089976698300017746}.

\bibitem[Ay et~al.(2015)Ay, Jost, Le, and Schwachhofer]{ay2015sufficient}
Nihat Ay, Jurgen Jost, Hong~Van Le, and Lorenz Schwachhofer.
\newblock Information geometry and sufficient statistics.
\newblock \emph{Probability Theory and Related Fields}, 162\penalty0
  (1--2):\penalty0 327--364, 2015.
\newblock \doi{10.1007/s00440-014-0574-8}.
\newblock URL \url{https://doi.org/10.1007/s00440-014-0574-8}.

\bibitem[Berestovskii \& Guijarro(2000)Berestovskii and
  Guijarro]{berestovskii2000metric}
V.~N. Berestovskii and Luis Guijarro.
\newblock A metric characterization of riemannian submersions.
\newblock \emph{Annals of Global Analysis and Geometry}, 18\penalty0
  (6):\penalty0 577--588, 2000.
\newblock \doi{10.1023/A:1006683922481}.
\newblock URL \url{https://doi.org/10.1023/A:1006683922481}.

\bibitem[Bhatia(2007)]{bhatia2007positive}
Rajendra Bhatia.
\newblock \emph{Positive Definite Matrices}.
\newblock Princeton University Press, 2007.

\bibitem[Bonnabel \& Sepulchre(2010)Bonnabel and
  Sepulchre]{bonnabel2010fixedrank}
Silvere Bonnabel and Rodolphe Sepulchre.
\newblock Riemannian metric and geometric mean for positive semidefinite
  matrices of fixed rank.
\newblock \emph{SIAM Journal on Matrix Analysis and Applications}, 31\penalty0
  (3):\penalty0 1055--1070, 2010.
\newblock \doi{10.1137/080731347}.
\newblock URL \url{https://doi.org/10.1137/080731347}.

\bibitem[Boyd \& Vandenberghe(2004)Boyd and Vandenberghe]{boyd2004convex}
Stephen Boyd and Lieven Vandenberghe.
\newblock \emph{Convex Optimization}.
\newblock Cambridge University Press, 2004.
\newblock URL \url{https://web.stanford.edu/~boyd/cvxbook/}.

\bibitem[Csiszar(1975)]{csiszar1975idivergence}
Imre Csiszar.
\newblock {$I$-Divergence Geometry of Probability Distributions and
  Minimization Problems}.
\newblock \emph{The Annals of Probability}, 3\penalty0 (1):\penalty0 146--158,
  1975.
\newblock \doi{10.1214/aop/1176996454}.
\newblock URL \url{https://doi.org/10.1214/aop/1176996454}.

\bibitem[Csiszar \& Matus(2003)Csiszar and Matus]{csiszar2003projections}
Imre Csiszar and Frantisek Matus.
\newblock Information projections revisited.
\newblock \emph{IEEE Transactions on Information Theory}, 49\penalty0
  (6):\penalty0 1474--1490, 2003.
\newblock \doi{10.1109/TIT.2003.810633}.
\newblock URL \url{https://doi.org/10.1109/TIT.2003.810633}.

\bibitem[Davis et~al.(2007)Davis, Kulis, Jain, Sra, and
  Dhillon]{davis2007information}
Jason~V. Davis, Brian Kulis, Prateek Jain, Suvrit Sra, and Inderjit~S. Dhillon.
\newblock Information-theoretic metric learning.
\newblock In \emph{Proceedings of the 24th International Conference on Machine
  Learning}, pp.\  209--216. ACM, 2007.
\newblock \doi{10.1145/1273496.1273523}.
\newblock URL \url{https://doi.org/10.1145/1273496.1273523}.

\bibitem[Dempster(1972)]{dempster1972covariance}
Arthur~P. Dempster.
\newblock Covariance selection.
\newblock \emph{Biometrics}, 28\penalty0 (1):\penalty0 157--175, 1972.
\newblock \doi{10.2307/2528966}.

\bibitem[Duchi et~al.(2011)Duchi, Hazan, and Singer]{duchi2011adagrad}
John Duchi, Elad Hazan, and Yoram Singer.
\newblock Adaptive subgradient methods for online learning and stochastic
  optimization.
\newblock \emph{Journal of Machine Learning Research}, 12:\penalty0 2121--2159,
  2011.
\newblock URL \url{https://jmlr.org/papers/v12/duchi11a.html}.

\bibitem[George et~al.(2018)George, Laurent, Bouthillier, Ballas, and
  Vincent]{george2018ekfac}
Thomas George, C{\'e}sar Laurent, Xavier Bouthillier, Nicolas Ballas, and
  Pascal Vincent.
\newblock Fast approximate natural gradient descent in a {Kronecker}-factored
  eigenbasis.
\newblock In \emph{Advances in Neural Information Processing Systems},
  volume~31, 2018.
\newblock URL \url{https://arxiv.org/abs/1806.03884}.

\bibitem[Grone et~al.(1984)Grone, Johnson, S{\'a}, and
  Wolkowicz]{grone1984positive}
Robert Grone, Charles~R. Johnson, Eduardo~M. S{\'a}, and Henry Wolkowicz.
\newblock Positive definite completions of partial hermitian matrices.
\newblock \emph{Linear Algebra and its Applications}, 58:\penalty0 109--124,
  1984.
\newblock \doi{10.1016/0024-3795(84)90207-6}.

\bibitem[Guijarro \& Walschap(2011)Guijarro and
  Walschap]{guijarro2011submetries}
Luis Guijarro and Gerard Walschap.
\newblock Submetries vs. submersions.
\newblock \emph{Revista Matematica Iberoamericana}, 27\penalty0 (2):\penalty0
  605--619, 2011.
\newblock \doi{10.4171/RMI/648}.
\newblock URL \url{https://doi.org/10.4171/RMI/648}.

\bibitem[Gupta et~al.(2018)Gupta, Koren, and Singer]{gupta2018shampoo}
Vineet Gupta, Tomer Koren, and Yoram Singer.
\newblock Shampoo: Preconditioned stochastic tensor optimization.
\newblock In \emph{Proceedings of the 35th International Conference on Machine
  Learning}, volume~80 of \emph{Proceedings of Machine Learning Research}, pp.\
   1842--1850. PMLR, 2018.
\newblock URL \url{https://proceedings.mlr.press/v80/gupta18a.html}.

\bibitem[Kingma \& Ba(2015)Kingma and Ba]{kingma2014adam}
Diederik~P. Kingma and Jimmy Ba.
\newblock Adam: A method for stochastic optimization.
\newblock In \emph{International Conference on Learning Representations}, 2015.
\newblock URL \url{https://arxiv.org/abs/1412.6980}.
\newblock arXiv:1412.6980.

\bibitem[Kunstner et~al.(2019)Kunstner, Balles, and
  Hennig]{kunstner2019limitations}
Frederik Kunstner, Lukas Balles, and Philipp Hennig.
\newblock Limitations of the empirical fisher approximation for natural
  gradient descent.
\newblock In \emph{Advances in Neural Information Processing Systems},
  volume~32, 2019.
\newblock URL
  \url{https://proceedings.neurips.cc/paper/2019/hash/46a558d97954d0692411c861cf78ef79-Abstract.html}.

\bibitem[Lovric et~al.(2000)Lovric, Min-Oo, and Ruh]{lovric2000multivariate}
Miroslav Lovric, Maung Min-Oo, and Ernst~A. Ruh.
\newblock Multivariate normal distributions parametrized as a riemannian
  symmetric space.
\newblock \emph{Journal of Multivariate Analysis}, 74\penalty0 (1):\penalty0
  36--48, 2000.
\newblock \doi{10.1006/jmva.1999.1853}.
\newblock URL \url{https://doi.org/10.1006/jmva.1999.1853}.

\bibitem[Martens(2020)]{martens2020natural}
James Martens.
\newblock New insights and perspectives on the natural gradient method.
\newblock \emph{Journal of Machine Learning Research}, 21\penalty0
  (146):\penalty0 1--76, 2020.
\newblock URL \url{https://jmlr.org/papers/v21/17-678.html}.

\bibitem[Martens \& Grosse(2015)Martens and Grosse]{martens2015kfac}
James Martens and Roger Grosse.
\newblock Optimizing neural networks with {Kronecker}-factored approximate
  curvature.
\newblock In \emph{Proceedings of the 32nd International Conference on Machine
  Learning}, volume~37 of \emph{Proceedings of Machine Learning Research}, pp.\
   2408--2417. PMLR, 2015.
\newblock URL \url{https://proceedings.mlr.press/v37/martens15.html}.

\bibitem[Massart \& Absil(2020)Massart and Absil]{massart2020quotient}
Estelle Massart and P.-A. Absil.
\newblock Quotient geometry with simple geodesics for the manifold of
  fixed-rank positive-semidefinite matrices.
\newblock \emph{SIAM Journal on Matrix Analysis and Applications}, 41\penalty0
  (1):\penalty0 171--198, 2020.
\newblock \doi{10.1137/18M1231389}.
\newblock URL \url{https://doi.org/10.1137/18M1231389}.

\bibitem[Morwani et~al.(2024)Morwani, Shapira, Vyas, Malach, Kakade, and
  Janson]{morwani2024shampoo}
Depen Morwani, Itai Shapira, Nikhil Vyas, Eran Malach, Sham~M. Kakade, and
  Lucas Janson.
\newblock A new perspective on {Shampoo}'s preconditioner, 2024.
\newblock URL \url{https://arxiv.org/abs/2406.17748}.

\bibitem[O'Neill(1966)]{oneill1966submersion}
Barrett O'Neill.
\newblock The fundamental equations of a submersion.
\newblock \emph{Michigan Mathematical Journal}, 13\penalty0 (4):\penalty0
  459--469, 1966.
\newblock \doi{10.1307/mmj/1028999604}.
\newblock URL \url{https://doi.org/10.1307/mmj/1028999604}.

\bibitem[Pennec et~al.(2006)Pennec, Fillard, and Ayache]{pennec2006riemannian}
Xavier Pennec, Pierre Fillard, and Nicholas Ayache.
\newblock A riemannian framework for tensor computing.
\newblock \emph{International Journal of Computer Vision}, 66\penalty0
  (1):\penalty0 41--66, 2006.
\newblock \doi{10.1007/s11263-005-3222-z}.

\bibitem[Qu et~al.(2025)Qu, Gao, Hinder, Ye, and Zhou]{qu2025optimaldiagonal}
Zhaonan Qu, Wenzhi Gao, Oliver Hinder, Yinyu Ye, and Zhengyuan Zhou.
\newblock Optimal diagonal preconditioning.
\newblock \emph{Operations Research}, 73\penalty0 (3):\penalty0 1479--1495,
  2025.
\newblock \doi{10.1287/opre.2022.0592}.
\newblock URL \url{https://doi.org/10.1287/opre.2022.0592}.

\bibitem[Raskutti \& Mukherjee(2015)Raskutti and
  Mukherjee]{raskutti2015information}
Garvesh Raskutti and Sayan Mukherjee.
\newblock The information geometry of mirror descent.
\newblock \emph{IEEE Transactions on Information Theory}, 61\penalty0
  (3):\penalty0 1451--1457, 2015.
\newblock \doi{10.1109/TIT.2015.2391243}.
\newblock URL \url{https://arxiv.org/abs/1310.7780}.

\bibitem[Rockafellar(1970)]{rockafellar1970convex}
R.~Tyrrell Rockafellar.
\newblock \emph{Convex Analysis}.
\newblock Princeton University Press, 1970.

\bibitem[Tropp(2015)]{tropp2015matrix}
Joel~A. Tropp.
\newblock An introduction to matrix concentration inequalities.
\newblock \emph{Foundations and Trends in Machine Learning}, 8\penalty0
  (1--2):\penalty0 1--230, 2015.
\newblock \doi{10.1561/2200000048}.
\newblock URL \url{https://arxiv.org/abs/1501.01571}.

\bibitem[Tumpach \& Larotonda(2024)Tumpach and Larotonda]{tumpach2024totally}
Alice~Barbara Tumpach and Gabriel Larotonda.
\newblock Totally geodesic submanifolds in the manifold {SPD} of symmetric
  positive-definite real matrices.
\newblock \emph{Information Geometry}, 7\penalty0 (S2):\penalty0 913--942,
  2024.
\newblock \doi{10.1007/s41884-024-00146-z}.
\newblock URL \url{https://doi.org/10.1007/s41884-024-00146-z}.

\bibitem[Vandenberghe \& Boyd(1996)Vandenberghe and
  Boyd]{vandenberghe1996semidefinite}
Lieven Vandenberghe and Stephen Boyd.
\newblock Semidefinite programming.
\newblock \emph{SIAM Review}, 38\penalty0 (1):\penalty0 49--95, 1996.
\newblock \doi{10.1137/1038003}.

\bibitem[Vandereycken et~al.(2013)Vandereycken, Absil, and
  Vandewalle]{vandereycken2013fixedrank}
Bart Vandereycken, P.-A. Absil, and Stefan Vandewalle.
\newblock A riemannian geometry with complete geodesics for the set of positive
  semidefinite matrices of fixed rank.
\newblock \emph{IMA Journal of Numerical Analysis}, 33\penalty0 (2):\penalty0
  481--514, 2013.
\newblock \doi{10.1093/imanum/drs006}.
\newblock URL \url{https://doi.org/10.1093/imanum/drs006}.

\bibitem[Yu et~al.(2015)Yu, Wang, and Samworth]{yu2015davis}
Yi~Yu, Tengyao Wang, and Richard~J. Samworth.
\newblock A useful variant of the davis--kahan theorem for statisticians.
\newblock \emph{Biometrika}, 102\penalty0 (2):\penalty0 315--323, 2015.
\newblock \doi{10.1093/biomet/asv008}.
\newblock URL \url{https://arxiv.org/abs/1405.0680}.

\end{thebibliography}

\appendix
\section{Proofs for Metric and Split-Hadamard Reduction}
\label{app:metric-proofs}

\begin{proof}[Proof of \cref{thm:metric-characterization}]
Assume (A).  If \(z\in X\) and \(r=d_X(x,z)\), ball surjectivity implies
\(d_Y(q(x),q(z))\le d_X(x,z)\); hence \(q\) is \(1\)-Lipschitz.  For fixed
\(y\), set \(\delta=d_Y(q(x),y)\).  Since
\(y\in\overline B_Y(q(x),\delta)\), submetry gives
\(z\in\overline B_X(x,\delta)\) with \(q(z)=y\).  Contraction yields
\[
  \delta\le d_X(x,z)\le\delta,
\]
proving (B), including attainment.

Assume (B).  Taking \(y=q(z)\) gives
\[
  d_Y(q(x),q(z))
  =\min_{q(w)=q(z)}d_X(x,w)
  \le d_X(x,z),
\]
so \(q\) is \(1\)-Lipschitz and maps the ambient ball into the visible ball.
If \(y\) lies in the visible ball, an attaining fiber minimizer lies in the
ambient ball.  Thus (B) implies (A).

Assume (B), fix \(z\), and put \(y=q(z)\).  Then
\(d_X(x,z)\ge d_Y(q(x),y)\), so monotonicity of \(\rho\) gives the lower
bound from the left side of \eqref{eq:universal-radial-reduction} to the
right.  An attaining fiber minimizer realizes equality for every \(y\),
proving the reverse inequality and hence (C).

Finally assume (C).  For \(r\ge0\) and fixed \(y\), use
\[
  \rho_r(t)=\begin{cases}0,&t\le r,\\+\infty,&t>r,\end{cases}
  \qquad
  \phi_y(u)=\begin{cases}0,&u=y,\\+\infty,&u\ne y.\end{cases}
\]
The reduced infimum is zero exactly when
\(y\in\overline B_Y(q(x),r)\), whereas the ambient infimum is zero exactly
when a point of \(q^{-1}(y)\) lies in \(\overline B_X(x,r)\).  Equality for
all \(y,r\) is precisely the submetry identity.  Thus (C) implies (A).
\end{proof}

\begin{proof}[Proof of \cref{thm:coherent-reduction}]
Assume (A).  If \(q(z)=y\), contraction gives
\[
  d_X(x,z)\ge d_Y(q(x),y).
\]
Because \(x=s_x(q(x))\) and \(s_x\) is isometric,
\[
  d_X(x,s_x(y))=d_Y(q(x),y),
\]
which proves (B).

Assume (B), fix \(y_1,y_2\), and put \(x_1=s_x(y_1)\).  Coherence gives
\(s_{x_1}=s_x\).  Applying (B) with anchor \(x_1\) and target \(y_2\),
\[
  d_X(s_x(y_1),s_x(y_2))
  =d_Y(q(x_1),y_2)=d_Y(y_1,y_2),
\]
so (B) implies (A).

Under (B), every \(z\) with \(q(z)=y\) has
\[
  \rho(d_X(x,z))+\phi(y)
  \ge \rho(d_Y(q(x),y))+\phi(y),
\]
and \(z=s_x(y)\) realizes equality.  Taking infima proves universal
reduction and canonical minimizer lifting, so (B) implies (C).

Assume (C), choose \(\rho(t)=t\), and let \(\phi\) be the indicator of a
singleton \(\{y\}\).  Universal reduction identifies the fiber infimum with
\(d_Y(q(x),y)\), while the canonical-minimizer clause says that \(s_x(y)\)
attains it.  Hence (C) implies (B).

The isometric-section property gives the reverse ball inclusion exactly as in
the proof of \cref{thm:metric-characterization}; hence \(q\) is a submetry.
For the proximal identity, apply universal reduction with
\(\rho(t)=t^2/(2\tau)\).  A proper lower-semicontinuous geodesically convex
function plus squared distance has a unique minimizer in complete CAT(0)
space.  The section lifts it with equal value.  Coherence then gives exact
iteration by induction.
\end{proof}

\section{Proof of the Split-Hadamard Theorem}
\label{app:split-proofs}

\begin{lemma}[Complete local isometries over a Hadamard base]
\label{lem:complete-local-isometry}
Let \(F:L\to B\) be a local Riemannian isometry from a connected complete
Riemannian manifold to a Hadamard manifold.  Then \(F\) is a global
Riemannian isometry.
\end{lemma}

\begin{proof}
Fix \(p\in L\).  Any geodesic \(\gamma:[0,1]\to B\) beginning at \(F(p)\)
has a unique local lift beginning at \(p\), obtained from a local inverse of
\(F\).  A lifted segment has the same speed as \(\gamma\).  If its maximal
interval ended at \(T\le1\), then for \(s,t\uparrow T\),
\[
  d_L(\widetilde\gamma_s,\widetilde\gamma_t)
  \le \operatorname{Len}(\widetilde\gamma|_{[s,t]})
  =\operatorname{Len}(\gamma|_{[s,t]})\longrightarrow0.
\]
Completeness gives a limit point in \(L\), and a local inverse at that point
extends the lift past \(T\).  Thus every such geodesic lifts over its entire
interval.  Since every point of the connected complete manifold \(B\) is
joined to \(F(p)\) by a minimizing geodesic, \(F\) is surjective.

The same continuation argument gives unique lifting for every piecewise
smooth path and for homotopies through such paths; equivalently, a complete
local isometry is a covering map.  The Hadamard manifold \(B\) is simply
connected, so its connected covering \(L\) has one sheet.  Hence \(F\) is a
diffeomorphism.  A bijective local Riemannian isometry preserves lengths in
both directions and is therefore a global Riemannian isometry.
\end{proof}

\begin{proof}[Proof of \cref{thm:split-hadamard}]
Because \(L_x\) is a horizontal integral leaf,
\(Dq|_{TL_x}\) is an isometry at every point, so \(q|_{L_x}\) is a local
Riemannian isometry.  The leaf is connected and complete by definition, and
the Hadamard base is simply connected.  Applying
\cref{lem:complete-local-isometry} shows directly that
\(q|_{L_x}:L_x\to B\) is a global Riemannian isometry.

The complete totally geodesic leaf \(L_x\) is geodesically convex in \(M\).
Indeed, Hopf--Rinow gives an intrinsic minimizing geodesic between any two
points of the leaf.  Total geodesy makes it an ambient geodesic, and the
Hadamard property makes that ambient geodesic the unique globally minimizing
one.  In particular, intrinsic and ambient distances agree on \(L_x\).
Any horizontal section through \(x\) lies in the same maximal integral leaf
and equals the inverse \(s_x\) of this global isometry.

For every piecewise smooth ambient curve \(\eta\), Riemannian-submersion
contraction gives
\[
  \operatorname{Len}(q\circ\eta)\le\operatorname{Len}(\eta),
\]
and hence \(d_B(q(y),q(z))\le d_M(y,z)\).  If \(q(z)=b\),
\[
  d_M(x,z)\ge d_B(q(x),b),
\]
while \(s_x(b)\) attains equality.  Suppose another \(z\) also attains it,
and let \(\eta\) be the ambient minimizing geodesic from \(x\) to \(z\).
Every inequality in
\[
  d_B(q(x),b)
  \le\operatorname{Len}(q\circ\eta)
  \le\operatorname{Len}(\eta)
  =d_M(x,z)
\]
is an equality.  The orthogonal decomposition
\(\dot\eta=\dot\eta^{\mathcal H}+\dot\eta^{\mathcal V}\) then forces
\(\dot\eta^{\mathcal V}=0\) everywhere.  The geodesic stays in \(L_x\) and
ends at \(s_x(b)\), proving uniqueness.  The ball identity and radial
reduction now follow from \cref{thm:metric-characterization}.

Proximal commutation follows from radial reduction with squared distance and
the uniqueness of proximal points on a Hadamard manifold.  Finally, for
\(v\in T_zM\),
\[
  \dd(\phi\circ q)_z[v]
  =h_{q(z)}(\operatorname{grad}_B\phi(q(z)),Dq(z)[v]).
\]
This vanishes on \(\mathcal V_z\); on \(\mathcal H_z\), \(Dq\) is an
isometry.  Riesz representation gives \eqref{eq:split-gradient}.  Since
\(Ds_x\) is the inverse horizontal isometry, the chain rule shows that
\(z=s_x\circ y\) solves the lifted equation on every interval on which the
visible solution \(y\) exists.  Local Lipschitz continuity gives uniqueness
of maximal solutions by the standard ODE theorem.
\end{proof}

\begin{proof}[Proof of \cref{cor:split-coherent}]
The sections are isometric.  If \(z=s_x(y)\), then \(z\) and \(x\) lie on
the same maximal horizontal leaf, so both \(s_z\) and \(s_x\) invert the
restriction of \(q\) to that leaf.  Thus \(s_z=s_x\), proving coherence.
The unique shortest-lift statement gives rigidity.
\end{proof}

\section{Proofs for SPD Compression and Completion}
\label{app:spd-proofs}

\begin{proof}[Proof of \cref{prop:reference-necessary}]
Choose a direct-sum decomposition \(E=W\oplus W'\) with \(W'\ne\{0\}\).
For \(t>0\), let \(T_t\) act as the identity on \(W\) and as multiplication
by \(t\) on \(W'\).  If a selected positive form \(P\) were invariant under
every such map, then for nonzero \(n\in W'\),
\[
  P(n,n)=P(T_tn,T_tn)=t^2P(n,n)
  \qquad\forall t>0,
\]
contradicting \(P(n,n)>0\).  Hence no invariant selector exists without a
reference or equivalent off-visible structure.
\end{proof}

\begin{proof}[Proof of \cref{thm:spd-compression}]
Fix \(P\in\SPD^d\), write \(R=A^\top PA\), and let \(H=H^\top\) be a
tangent at \(P\).  Define
\[
  C=P^{1/2}AR^{-1/2},
  \qquad
  X=P^{-1/2}HP^{-1/2}.
\]
Then \(C^\top C=I_m\) and
\[
  R^{-1/2}A^\top HA R^{-1/2}=C^\top XC.
\]
Complete \(C\) to an orthogonal matrix \([C,N]\).  The exact block identity
is
\begin{equation}
  \fro X^2
  =\fro{C^\top XC}^2
  +2\fro{N^\top XC}^2
  +\fro{N^\top XN}^2.
  \label{eq:block-frobenius-identity}
\end{equation}
Therefore
\[
  \norm{D\pi_A(P)[H]}_{R,\AI}
  =\fro{C^\top XC}
  \le\fro X
  =\norm H_{P,\AI}.
\]
Integrating along curves proves global contraction
\eqref{eq:spd-contraction}.

Congruence by \(P_0^{-1/2}\) in the ambient cone and by \(R_0^{-1/2}\) in
the visible cone reduces the completion problem to
\[
  \min\{\dai(I,X):U_0^\top XU_0=\widetilde R\},
  \qquad
  \widetilde R=R_0^{-1/2}RR_0^{-1/2}.
\]
The candidate
\[
  X_\star=I+U_0(\widetilde R-I)U_0^\top
\]
is positive definite, has compression \(\widetilde R\), and satisfies
\[
  \dai(I,X_\star)=\fro{\log\widetilde R}=\dai(I,\widetilde R).
\]
Contraction proves it is a minimizer.

For uniqueness, suppose feasible \(X\) attains equality and write
\(L=\log X\).  The compressed curve
\(Y_t=U_0^\top\exp(tL)U_0\) joins \(I\) to \(\widetilde R\).  Hence
\[
  \dai(I,\widetilde R)
  \le\operatorname{Len}(Y)
  \le\fro L
  =\dai(I,X)
  =\dai(I,\widetilde R).
\]
Equality of the continuous speed inequality at \(t=0\), together with
\eqref{eq:block-frobenius-identity}, forces
\(N^\top LU_0=0\) and \(N^\top LN=0\).  Thus
\(L=U_0L_0U_0^\top\), and the endpoint constraint gives
\(\exp(L_0)=\widetilde R\).  Therefore \(X=X_\star\).  Undoing congruence
gives \eqref{eq:canonical-completion} and
\eqref{eq:spd-shortest-completion}.

For \(R,S\in\SPD^m\), the normalized lifts are
\(\widetilde R\oplus I\) and \(\widetilde S\oplus I\).  Their ambient AIRM
distance is the visible AIRM distance, and their geodesic keeps the identity
block fixed.  Thus the section is isometric and totally geodesic.  The ball
identity follows by combining contraction with the equal-distance lift.

At a general \(P\), let
\(L_P(h)=D\Psi_{P,A}(A^\top PA)[h]\).  The section identity and isometry give
\[
  D\pi_A(P)[L_P(h)]=h,
  \qquad
  \norm{L_P(h)}_{P,\AI}=\norm h_{A^\top PA,\AI}.
\]
If \(V\in\ker D\pi_A(P)\), then \(L_P(h)+sV\) is another lift of \(h\).
Infinitesimal contraction gives
\(\norm h^2\le\norm{L_P(h)+sV}^2\) for all \(s\), with equality at zero.
Differentiation shows \(L_P(h)\perp V\).  Surjectivity and dimension counting
identify the image of \(L_P\) with the full horizontal space.  This proves
the Riemannian-submersion and horizontal-lift statements.  The images of the
sections through all \(P\) are complete totally geodesic maximal horizontal
integral leaves; hence the horizontal distribution is integrable and
\(\pi_A\) is split-Hadamard.

For the minimax statement, work in the normalized splitting and choose
symmetric \(H\in\R^{(d-m)\times(d-m)}\) with \(\fro H=1\).  The curve
\[
  X_s=\widetilde R\oplus\exp(sH)
\]
stays in the visible fiber and satisfies
\(\dai(X_0,X_s)=|s|\).  This proves infinite diameter.  The two points
\(X_{-B},X_B\) lie in the bounded fiber and are distance \(2B\) apart, so
every center has worst-case distance at least \(B\).  The canonical point
\(X_0\) has worst-case distance at most \(B\).  If another point attained
\(B\), equality in the triangle inequality between the endpoints would make
it their midpoint.  AIRM is uniquely geodesic, so the midpoint is uniquely
\(X_0\).

Finally, for any admissible \((P_0,A)\), every \(P\) with
\(R=A^\top PA\) obeys
\[
  \dai(P_0,P)\ge\dai(A^\top P_0A,R),
\]
and the canonical lift realizes equality.  Monotonicity of \(\rho\) proves
both inequalities in \eqref{eq:spd-decision-reduction}.  The same construction
lifts every reduced minimizer.  Under strict monotonicity, replacing a
noncanonical full point by the strictly nearer canonical point lowers the
objective, proving the final correspondence.
\end{proof}

\begin{proof}[Proof of \cref{cor:visible-update}]
For \(\alpha=Ac\),
\[
  A^\top u_P=-A^\top PA c=-Rc,
\]
so all feasible updates have the same quotient class.

For the converse, use a \(P_0\)-orthogonal splitting of the cotangent space
into the visible span and its complement, and choose coordinates in which
\(\alpha\) has visible coordinate \(c_0\ne0\).  A feasible completion has
blocks
\[
  P=\begin{pmatrix}R&B\\B^\top&C\end{pmatrix}.
\]
The invisible component of \(-P\alpha\) is \(-B^\top c_0\).  Given the
coordinate vector \(z\) of any desired invisible drift, take
\[
  B=-\frac{c_0z^\top}{c_0^\top c_0}.
\]
Then \(-B^\top c_0=z\).  Choosing
\[
  C=B^\top R^{-1}B+tI,
  \qquad t>0,
\]
makes the Schur complement positive, hence \(P\succ0\).  This realizes every
invisible drift.  The final claim follows from the block-diagonal form and
uniqueness of the canonical completion.
\end{proof}

\begin{proof}[Proof of \cref{cor:gaussian-kl}]
Normalize to \(P_0=I\) and an orthogonal frame adapted to \(U_0\).  Every
feasible covariance has blocks
\[
  X=\begin{pmatrix}\widetilde R&C\\C^\top&D\end{pmatrix},
  \qquad
  S=D-C^\top\widetilde R^{-1}C\succ0.
\]
Up to constants independent of \((C,S)\), twice the forward Gaussian KL is
\begin{align*}
  \tr X-\log\det X
  &=\tr\widetilde R-\log\det\widetilde R
  +\tr S-\log\det S
  +\tr(C^\top\widetilde R^{-1}C).
\end{align*}
The last term is uniquely minimized at \(C=0\).  The scalar inequality
\(s-\log s\ge1\) applied to the eigenvalues of \(S\) shows that the middle
term is uniquely minimized at \(S=I\).  Hence
\(X=\widetilde R\oplus I\), the normalized canonical completion.  Congruence
returns \(\Psi_{P_0,A}(R)\).
\end{proof}

\begin{proof}[Proof of \cref{cor:prior-data}]
For any \(P\), put \(R=A^\top PA\).  Canonical completion gives
\[
  \dai(P_0,P)\ge\dai(R_0,R),
\]
with equality exactly at \(P=\Psi_{P_0,A}(R)\).  The full problem therefore
reduces exactly to
\[
  \min_{R\in\SPD^m}
  (1-\tau)\dai(R_0,R)^2+\tau\dai(R,\Rhat)^2.
\]
Let \(x=\dai(R_0,R)\), \(y=\dai(R,\Rhat)\), and
\(D=\dai(R_0,\Rhat)\).  The triangle inequality gives \(x+y\ge D\), while
\[
  (1-\tau)x^2+\tau y^2
  =\tau(1-\tau)(x+y)^2+((1-\tau)x-\tau y)^2.
\]
The lower bound \(\tau(1-\tau)D^2\) is attained only when \(R\) lies on the
unique AIRM geodesic from \(R_0\) to \(\Rhat\) with
\(x=\tau D\), \(y=(1-\tau)D\).  Thus
\(R=R_0\#_\tau\Rhat\), and canonical completion proves
\eqref{eq:prior-data-solution} and the distance identities.

For robustness, geodesic convexity of distance in a Hadamard manifold gives
\[
  \dai(R_0\#_\tau R_1,R_0\#_\tau R_2)
  \le\tau\dai(R_1,R_2).
\]
The completion section is isometric, so the same bound holds for the full
solutions.
\end{proof}

\section{Proofs for Rank-Stratified Geometry}
\label{app:stratified-proofs}

\begin{proof}[Proof of \cref{thm:rank-stratification}]
The completion formula is gauge invariant because replacing
\((A,R)\) by \((AB,B^\top RB)\) leaves
\(P_0AR_0^{-1}(R-R_0)R_0^{-1}A^\top P_0\) unchanged.  Moreover,
\begin{equation}
  \Psi_{P_0,A}(R)-P_0
  =P_0AR_0^{-1}(R-R_0)R_0^{-1}A^\top P_0.
  \label{eq:rank-factorization-proof}
\end{equation}
If \(R-R_0\) is invertible, this difference has rank \(m\).

Conversely, let \(P\in\mathfrak S_m(P_0)\), put \(H=P-P_0\), and define
\(\mathcal W=P_0^{-1}\range(H)\).  Choose full-column-rank \(A\) spanning
\(\mathcal W\).  Then \(P_0A\) spans \(\range(H)\).  Since \(H\) is
symmetric of rank \(m\), there is an invertible symmetric \(C\) such that
\[
  H=P_0ACA^\top P_0.
\]
For example, with \(B=P_0A\), take
\[
  C=(B^\top B)^{-1}B^\top H B(B^\top B)^{-1}.
\]
The identity follows because the range and co-range of \(H\) both equal
\(\range(B)\).  Setting \(R=A^\top PA\) and \(R_0=A^\top P_0A\) gives
\[
  R-R_0=A^\top H A=R_0CR_0,
\]
which is invertible, and \eqref{eq:rank-factorization-proof} reconstructs
\(P\).

The subspace \(\mathcal W\) is intrinsic.  Two frames for it differ by a
unique \(B\in GL(m)\), and their compressed matrices differ by congruence.
This proves bijectivity.  Locally on a fixed-rank stratum, the nonzero
spectrum of \(P_0^{-1/2}(P-P_0)P_0^{-1/2}\) remains separated from zero.
Spectral functional calculus makes its range projector smooth; local frames
and their restrictions then give a smooth inverse.  Thus \(\Gamma_{m,P_0}\)
is a diffeomorphism.

Every SPD matrix lies in exactly one rank stratum.  A limit of rank-\(m\)
symmetric differences has rank at most \(m\).  Conversely, if
\(\rank(P-P_0)=r<m\), whiten by \(P_0\), choose \(m-r\) orthonormal kernel
directions of the whitened difference, and add arbitrarily small nonzero
eigenvalues in those directions.  The perturbation stays SPD for sufficiently
small magnitudes, has rank \(m\), and converges to \(P\).  This proves
\eqref{eq:strata-closure}.
\end{proof}

\begin{proof}[Proof of \cref{thm:stratified-line-element}]
Under \(A^\top P_0A=I\), put \(U=P_0^{1/2}A\).  Whitening sends the base
point and the differential represented by \(\xi_i=(H_i,K_i)\) to
\[
  \Gamma\longleftrightarrow
  \begin{pmatrix}R&0\\0&I\end{pmatrix},
  \qquad
  D\widetilde\Gamma[\xi_i]\longleftrightarrow
  \begin{pmatrix}
    H_i&(R-I)K_i^\top\\
    K_i(R-I)&0
  \end{pmatrix}.
\]
Using
\(
\langle X,Y\rangle_{P,\AI}=\tr(P^{-1}XP^{-1}Y)
\), direct block multiplication for directions \(i,j\) gives
\[
  \tr(R^{-1}H_iR^{-1}H_j)
  +2\tr(K_i(R-I)R^{-1}(R-I)K_j^\top),
\]
which is \eqref{eq:completion-pair-kernel}.  Setting \(i=j\) proves
\eqref{eq:stratified-line-element}.  Factoring its two blocks gives
\eqref{eq:completion-feature}.  The pair matrix is therefore a Gram matrix,
which proves its positive semidefiniteness, the rank identity
\eqref{eq:completion-gram-rank}, and the orthogonality criterion
\eqref{eq:completion-pair-orthogonality}.

The differential vanishes exactly when \(H=0\) and \(K(R-I)=0\), proving
\eqref{eq:stratified-kernel}.  Hence the pullback is positive definite on
horizontal quotient tangents exactly when \(R-I\) is invertible; this is the
regular domain of \cref{thm:rank-stratification}.

The eigenvalues of \(\Xiop(R)\) are
\((\lambda_i(R)-1)^2/\lambda_i(R)\).  Under the spectral assumptions,
\[
  \lambda_{\min}(\Xiop(R))\ge\frac{\zeta^2}{b},
  \qquad
  \lambda_{\max}(\Xiop(R))
  \le\max_{x\in[a,b]}\frac{(x-1)^2}{x}.
\]
Substitution into the exact line element proves
\eqref{eq:stratified-conditioning}.
\end{proof}

\begin{proof}[Proof of \cref{cor:reduced-completion-value}]
The shortest-completion identity \eqref{eq:spd-shortest-completion} gives
\[
  \min_{A^\top PA=R}\frac12\dai(P_0,P)^2
  =\frac12\dai(A^\top P_0A,R)^2.
\]
Under the normalization \(A^\top P_0A=I\), this is
\eqref{eq:reduced-completion-energy} and is independent of the horizontal
subspace coordinate.  Derivatives with respect to that coordinate at fixed
\(R\) consequently vanish.  The minimizer is the explicit section
\(P=\Psi_{P_0,A}(R)\), so evaluating it requires no further vertical
stationarity equation.  Differentiating that section instead measures the
motion of the minimizing ambient matrix, whose AIRM inner product is
\eqref{eq:completion-pair-kernel}; it is not the Hessian of the reduced value.
\end{proof}

\paragraph{Frozen--relaxed decomposition in the two-dimensional example.}
The cancellation mentioned in \cref{rem:value-versus-motion} can be made
explicit in fixed ambient coordinates.  Write
\[
  u(\theta)=(\cos\theta,\sin\theta)^\top,
  \qquad
  P=\begin{pmatrix}a&c\\c&b\end{pmatrix},
\]
and eliminate the constraint \(u(\theta)^\top Pu(\theta)=r\) as
\[
  a=
  \frac{r-2c\sin\theta\cos\theta-b\sin^2\theta}
       {\cos^2\theta}.
\]
For
\(F(b,c,\theta)=\tfrac12\fro{\log P}^2\), the hidden minimizer at
\(\theta=0\) is \((b,c)=(1,0)\).  Direct differentiation there gives
\[
  F_{\theta\theta}=\frac{2(r-1)\log r}{r},
  \qquad
  F_{\theta c}=-\frac{2\log r}{r},
  \qquad
  F_{cc}=\frac{2\log r}{r(r-1)},
\]
with \(F_{\theta b}=F_{bc}=0\) and \(F_{bb}=1\).  The hidden Hessian is
positive definite because \(\log r/(r-1)>0\).  Consequently the relaxed
second derivative is
\[
  F_{\theta\theta}
  -F_{\theta z}F_{zz}^{-1}F_{z\theta}
  =\frac{2(r-1)\log r}{r}
  -\frac{2(r-1)\log r}{r}
  =0,
  \qquad z=(b,c).
\]
Thus the frozen direct term and the nonzero Schur term cancel, consistently
with the constant reduced value.  The response
\(c'(0)=-F_{cc}^{-1}F_{c\theta}=r-1\) is the off-diagonal derivative of the
canonical completion in \cref{rem:value-versus-motion}.

\begin{proof}[Proof of \cref{prop:moving-reference}]
By definition, \(P_t=G_tQ_tG_t^\top\).  Differentiating and multiplying by
\(G_t^{-1}\) and \(G_t^{-\top}\) gives
\[
  G_t^{-1}\dot P_tG_t^{-\top}
  =\dot Q_t+\Omega_tQ_t+Q_t\Omega_t^\top
  =\mathfrak D_tQ_t.
\]
Congruence invariance of AIRM yields \eqref{eq:moving-reference}.

If \(G_t'=G_tO_t\), then
\[
  Q_t'=O_t^\top Q_tO_t,
  \qquad
  \Omega_t'=O_t^\top\Omega_tO_t+O_t^\top\dot O_t.
\]
Differentiating \(Q_t'\) and substituting these identities cancels all terms
containing \(O_t^\top\dot O_t\), leaving
\[
  \mathfrak D_t'Q_t'=O_t^\top(\mathfrak D_tQ_t)O_t.
\]
Orthogonal congruence preserves the displayed Frobenius norm.
\end{proof}

\begin{proof}[Proof of \cref{cor:rank-loss}]
After whitening, the nonzero singular values of
\[
  P_t-I=U_t(R_t-I)U_t^\top
\]
are exactly those of \(R_t-I\).  Singular values vary continuously with the
matrix.  If the limit has rank at most \(m-r\), at least \(r\) of the first
\(m\) singular values converge to zero.  The same therefore holds for
\(R_t-I\).
\end{proof}

\section{Proofs for Structured Expressivity}
\label{app:expressivity-proofs}

\begin{lemma}[Linear images of relative interiors]
\label{lem:linear-ri}
Let \(C\) be a nonempty convex set in a finite-dimensional vector space and
let \(L\) be linear.  Then
\begin{equation}
  L(\ri C)=\ri L(C).
  \label{eq:linear-ri}
\end{equation}
\end{lemma}

\begin{proof}
Let \(x\in\ri C\), set \(y=Lx\), and choose any \(z=Lx_z\in L(C)\).
The relative-interior extension property gives \(\eta>0\) with
\(x+\eta(x-x_z)\in C\).  Applying \(L\) gives
\(y+\eta(y-z)\in L(C)\), proving \(y\in\ri L(C)\).

Conversely, let \(y\in\ri L(C)\), choose \(x_0\in\ri C\), and set
\(z=Lx_0\).  The extension property in \(L(C)\) gives \(\eta>0\) and
\(x_1\in C\) satisfying \(Lx_1=y+\eta(y-z)\).  Then
\[
  x=\frac{x_1+\eta x_0}{1+\eta}
\]
belongs to \(\ri C\) and satisfies \(Lx=y\).
\end{proof}

\begin{proof}[Proof of \cref{thm:diagonal-expressivity}]
Let \(L:\R^d\to\Sym^m\) be \(L(p)=\sum_jp_jG_j\).  Since
\(\ri\R_+^d=\R_{++}^d\), \cref{lem:linear-ri} gives
\[
  L(\R_{++}^d)=\ri L(\R_+^d)=\ri\Kcal_{\mathrm{diag}}.
\]
This proves the feasibility criterion.  With \(y_j=\log p_j\), the objective
is \(\norm{y-y^0}_2\) and the constraint is
\(\sum_je^{y_j}G_j=R\).  A minimizing sequence is bounded in \(y\), hence
has a convergent subsequence whose limit remains feasible.  Thus the infimum
is attained.  Linear independence of the \(G_j\) gives uniqueness of the
coefficients.

The cone \(\Kcal_{\mathrm{diag}}\) is finitely generated and closed.  If
\(R\) lies outside it, closed-cone separation gives symmetric \(Y\) with
\(\tr(YR)<0\) and \(\tr(YG_j)\ge0\) for all \(j\).  If \(R\) lies on its
relative boundary, apply the supporting-hyperplane theorem in
\(S=\operatorname{span}\Kcal_{\mathrm{diag}}\).  It gives a linear
functional on \(S\), nonzero on \(S\), that vanishes at \(R\) and is
nonnegative on the cone.  Represent it by a symmetric \(Y\).  Then
\(\tr(YR)=0\) and \(\tr(YG_j)\ge0\) for every \(j\).  At least one of these
generator pairings is strictly positive: otherwise the functional would
vanish on the span of all generators, namely on \(S\).  This proves the
stated facial certificate and excludes ambient annihilators of the cone.
\end{proof}

\begin{proof}[Proof of \cref{thm:block-expressivity}]
Let \(C=\prod_b\PSD^{(b)}\) and
\[
  L((P_b)_b)=\sum_bA_b^\top P_bA_b.
\]
The relative interior of \(C\) is \(\prod_b\SPD^{(b)}\).  Therefore
\cref{lem:linear-ri} gives
\[
  L\!\left(\prod_b\SPD^{(b)}\right)
  =\ri L(C)=\ri\Kcal_{\mathrm{block}}.
\]
This proves strict block feasibility.  If a feasible sequence has bounded
objective, the generalized eigenvalues of each
\((P_b^{(n)},P_b^0)\) stay in a common compact interval
\([e^{-C_0},e^{C_0}]\).  A subsequence converges to positive definite blocks,
and the linear equality is closed, proving attainment.

For each block, the image cone
\[
  C_b=\{A_b^\top P_bA_b:P_b\succeq0\}
\]
is exactly the cone of PSD matrices whose range is contained in
\(\range(A_b^\top)\), and is therefore closed.  To see that their finite sum
is closed, let \(X_n=\sum_bX_{b,n}\) converge with \(X_{b,n}\in C_b\).
Every summand is PSD, so
\(0\le\tr(X_{b,n})\le\tr(X_n)\).  The summands are consequently bounded;
passing to common subsequences gives \(X_{b,n}\to X_b\in C_b\) and
\(X=\sum_bX_b\).  Hence
\(\Kcal_{\mathrm{block}}=\sum_bC_b\) is closed.

If \(R\notin\Kcal_{\mathrm{block}}\), separation from the closed
cone gives symmetric \(Y\) with \(\tr(YR)<0\) and
\[
  \tr\!\left(Y\sum_bA_b^\top P_bA_b\right)\ge0
  \quad\text{for all }P_b\succeq0.
\]
Blockwise duality of the PSD cone makes this equivalent to
\(A_bYA_b^\top\succeq0\) for every \(b\).  If \(R\) lies on the relative
boundary, take a supporting functional that is nonzero on
\(S=\operatorname{span}\Kcal_{\mathrm{block}}\).  It gives the same
blockwise condition with \(\tr(YR)=0\).  Some
\(A_bYA_b^\top\) must be nonzero: if every one vanished, the functional
would vanish on every image \(A_b^\top P_bA_b\), hence on \(S\).
\end{proof}

\section{Proofs for Deterministic and Statistical Recovery}
\label{app:recovery-proofs}

\begin{proof}[Proof of \cref{thm:joint-perturbation}]
Set
\[
  P=\Cpl(U,R),
  \qquad
  P_1=\Cpl(U,S),
  \qquad
  Q=\Cpl(V,S).
\]
At fixed \(U\), the completion section is isometric, so
\(\dai(P,P_1)=\dai(R,S)\).  Let \(U_t\) be the horizontal minimizing
Grassmann geodesic from \(U\) to \(V\), and keep \(S\) fixed in the parallel
frame.  The exact line element gives
\[
  \norm{\dot Q_t}_{Q_t,\AI}^2
  =2\tr(K_t\Xiop(S)K_t^\top)
  \le2\chi(S)^2\fro{K_t}^2.
\]
Therefore
\[
  \dai(P_1,Q)
  \le\sqrt2\chi(S)\int_0^1\fro{K_t}\,\dd t
  =\sqrt2\chi(S)\dgr(W,\widetilde W).
\]
The triangle inequality gives the claimed bound with \(\chi(S)\).  Reversing
the order, first rotating with \(R\) and then changing visible geometry at
fixed \(V\), gives the bound with \(\chi(R)\).  Taking the smaller proves
\eqref{eq:joint-perturbation}.
\end{proof}

\begin{proof}[Proof of \cref{thm:known-subspace-recovery}]
Let \(Y_i=R^{-1/2}X_i\).  Then
\[
  \E[Y_iY_i^\top]=I,
  \qquad
  \norm{Y_i}_2\le L,
\]
and
\[
  \widehat S=R^{-1/2}\Rhat R^{-1/2}
  =\frac1n\sum_{i=1}^nY_iY_i^\top.
\]
The \(Y_i\) are i.i.d.  For fixed unit \(u\), the variables
\(Z_i=(u^\top Y_i)^2\) satisfy \(0\le Z_i\le L^2\) and
\(\E Z_i=1\).  Hoeffding's inequality gives
\[
  \Pr\left\{
  \left|\frac1n\sum_iZ_i-1\right|\ge t
  \right\}
  \le2\exp\left(-\frac{2nt^2}{L^4}\right).
\]
Let \(\mathcal N\) be a \(1/4\)-net of the unit sphere in \(\R^m\), with
\(|\mathcal N|\le9^m\), and take \(t=\varepsilon/4\).  A union bound and the
sample-size assumption show that, with probability at least \(1-\delta\),
\[
  |u^\top(\widehat S-I)u|\le\varepsilon/4
  \qquad\forall u\in\mathcal N.
\]
For symmetric \(H\), a \(1/4\)-net obeys
\[
  \opnorm H\le2\sup_{u\in\mathcal N}|u^\top Hu|.
\]
Indeed, if unit \(x\) attains the operator norm and
\(\norm{x-u}\le1/4\), then
\(|x^\top Hx-u^\top Hu|\le\frac12\opnorm H\).  Hence
\(\opnorm{\widehat S-I}\le\varepsilon/2\).  For
\(0<\varepsilon\le1\),
\[
  e^{-\varepsilon}I
  \preceq(1-\varepsilon/2)I
  \preceq\widehat S
  \preceq(1+\varepsilon/2)I
  \preceq e^\varepsilon I.
\]
Congruence by \(R^{1/2}\) gives the Loewner sandwich.  Its generalized
log-eigenvalues have magnitude at most \(\varepsilon\), so
\(\dai(R,\Rhat)\le\sqrt m\,\varepsilon\).  Fixed-channel completion is
isometric, proving \eqref{eq:known-subspace-bound}.
\end{proof}

\begin{lemma}[Spectral-gap subspace perturbation]
\label{lem:spectral-subspace}
Let \(\Sigmahat=\Sigma+E\),
\(\delta=\opnorm E<\gamma\), and let \(V\) span the leading
\(m\)-dimensional eigenspace of \(\Sigmahat\).  Then
\begin{equation}
  \opnorm{U_\perp^\top V}
  \le s_\delta=\frac{\delta}{\gamma-\delta}.
  \label{eq:sin-theta}
\end{equation}
If \(\delta<\gamma/2\), then
\begin{equation}
  \dgr(W,\widehat W)
  \le\sqrt m\,\arcsin(s_\delta),
  \label{eq:grassmann-stat}
\end{equation}
and after orthogonal Procrustes alignment,
\begin{equation}
  \opnorm{V-U}\le\sqrt2\,s_\delta.
  \label{eq:frame-stat}
\end{equation}
\end{lemma}

\begin{proof}
In the orthogonal basis \([U,U_\perp]\), write
\[
  \Sigma=\begin{pmatrix}A&0\\0&C\end{pmatrix},
  \qquad A\succeq\lambda_mI,
  \qquad C\preceq\lambda_{m+1}I.
\]
Let \(\widehat\Lambda\) contain the leading eigenvalues of \(\Sigmahat\).
Weyl's inequality gives
\(\widehat\Lambda\succeq(\lambda_m-\delta)I\).  With
\(Y=U_\perp^\top V\), projection of
\(\Sigmahat V=V\widehat\Lambda\) gives the Sylvester equation
\[
  Y\widehat\Lambda-CY=F,
  \qquad F=U_\perp^\top EV.
\]
The spectra are separated by at least \(\gamma-\delta\), and
\[
  Y=\int_0^\infty e^{tC}Fe^{-t\widehat\Lambda}\,\dd t.
\]
Thus
\[
  \opnorm Y
  \le\int_0^\infty e^{-t(\gamma-\delta)}\delta\,\dd t
  =\frac{\delta}{\gamma-\delta}.
\]
The singular values of \(U_\perp^\top V\) are the sines of the principal
angles, proving \eqref{eq:grassmann-stat}.  In principal frames,
\[
  \opnorm{V-U}
  =\max_i2\sin(\theta_i/2)
  \le\sqrt2\max_i\sin\theta_i
  \le\sqrt2s_\delta,
\]
which proves \eqref{eq:frame-stat}.
\end{proof}

\begin{proof}[Proof of \cref{thm:unknown-subspace-recovery}]
Align \(V\) to \(U\) as in \cref{lem:spectral-subspace}, and conjugate
\(\Rhat\) by the same orthogonal alignment.  Denote the aligned matrix by
\(S\); the completed estimator is unchanged.  Then
\begin{align*}
  \opnorm{S-R}
  &\le\opnorm{V^\top(\Sigmahat-\Sigma)V}
  +\opnorm{V^\top\Sigma V-U^\top\Sigma U}\\
  &\le\delta+2\lambda_1\opnorm{V-U}\\
  &\le\delta+2\sqrt2\lambda_1s_\delta
  =\varepsilon_{\mathrm{vis}}.
\end{align*}
Since \(R\succeq\lambda_mI\),
\[
  \opnorm{R^{-1/2}(S-R)R^{-1/2}}
  \le\eta_\delta<1.
\]
Therefore
\[
  (1-\eta_\delta)R\preceq S\preceq(1+\eta_\delta)R,
\]
and
\begin{equation}
  \dai(R,S)\le\sqrt m[-\log(1-\eta_\delta)].
  \label{eq:aligned-visible-bound}
\end{equation}
Weyl's inequality places the eigenvalues of \(S\) in
\([\lambda_m-\delta,\lambda_1+\delta]\), so
\(\chi(S)\le\omega_\delta\).  Apply
\cref{thm:joint-perturbation,lem:spectral-subspace} and
\eqref{eq:aligned-visible-bound}:
\begin{align*}
  \dai(P,\Phat)
  &\le\dai(R,S)+\sqrt2\chi(S)\dgr(W,\widehat W)\\
  &\le\sqrt m[-\log(1-\eta_\delta)]
  +\sqrt2\omega_\delta\sqrt m\arcsin(s_\delta).
\end{align*}
This is \eqref{eq:unknown-subspace-bound}.
\end{proof}

\begin{lemma}[Ambient second-moment concentration]
\label{lem:ambient-concentration}
Under the assumptions of \cref{cor:unknown-subspace-sample}, with probability
at least \(1-\alpha\),
\[
  \opnorm{\Sigmahat-\Sigma}\le\delta_n(\alpha).
\]
\end{lemma}

\begin{proof}
Let \(\mathcal N\) be a \(1/4\)-net of the Euclidean unit sphere with
\(|\mathcal N|\le9^d\).  For fixed \(u\in\mathcal N\), the i.i.d. variables
\(Z_i=(u^\top X_i)^2\) lie in \([0,L^2]\).  Hoeffding and a union bound give
\[
  \Pr\left\{
  \sup_{u\in\mathcal N}|u^\top(\Sigmahat-\Sigma)u|\ge t
  \right\}
  \le2\cdot9^d\exp\left(-\frac{2nt^2}{L^4}\right).
\]
The net estimate used above gives
\(\opnorm{\Sigmahat-\Sigma}\le2t\).  Taking
\(t=\delta_n(\alpha)/2\) makes the failure probability at most \(\alpha\).
\end{proof}

\begin{proof}[Proof of \cref{cor:unknown-subspace-sample}]
On the event from \cref{lem:ambient-concentration}, all assumptions of
\cref{thm:unknown-subspace-recovery} hold with
\(\delta=\delta_n(\alpha)\).  Apply that theorem.
\end{proof}

\end{document}